\documentclass{article}

\PassOptionsToPackage{numbers,compress}{natbib}
\usepackage[preprint]{neurips_2022}
\usepackage[utf8]{inputenc} 
\usepackage[T1]{fontenc}    

\usepackage{hyperref}       
\usepackage{url}            
\usepackage{tabularx,booktabs}       
\usepackage{multirow}
\usepackage{amsfonts,amsmath,amssymb,graphicx,amsthm,bm,bbm,enumerate}
\usepackage{algorithm}
\usepackage{algorithmic}
\usepackage{color,soul}
\usepackage{nicefrac}       
\usepackage{microtype}  
\usepackage{xcolor,caption}
\usepackage{float}

\theoremstyle{definition}

\newtheorem*{theorem*}{Theorem}
\newtheorem{theorem}{Theorem}
\newtheorem{definition}{Definition}
\newtheorem{proposition}{Proposition}

\newtheorem{lemma}{Lemma}

\newtheorem{remark}{Remark}
\newtheorem*{remark*}{Remark}

\title{
Conformalized Fairness via Quantile Regression
}

\author{%
Meichen Liu$^1$, Lei Ding$^1$, Dengdeng Yu$^2$, Wulong Liu$^3$, \\ \textbf{  Linglong Kong$^{1}$\thanks{Co-Corresponding Authors}, Bei Jiang$^{1*}$} 
 \\
$^1$Department of Mathematical and Statistical Sciences, University of Alberta\\
$^2$Department of Mathematics, University of Texas at Arlington \\
$^3$ Huawei Noah's Ark Lab Canada \\
\texttt{\{meichen1,lding1,lkong,bei1\}@ualberta.ca} \\
\texttt{\{dengdeng.yu\}@uta.edu} \\
\texttt{\{liuwulong\}@huawei.com}
}

\begin{document}

\maketitle

\begin{abstract}

Algorithmic fairness has received increased attention in socially sensitive domains.
While rich literature on mean fairness has been established, research on quantile fairness remains sparse but vital. 
To fulfill great needs and advocate the significance of quantile fairness, we propose a novel framework to learn a real-valued quantile function under the fairness requirement of \emph{Demographic Parity} with respect to sensitive attributes, such as race or gender, and thereby derive a reliable \emph{fair} prediction interval. 
Using optimal transport and functional synchronization techniques, we establish theoretical guarantees of distribution-free coverage and exact fairness for the induced prediction interval constructed by fair quantiles. 
A hands-on pipeline is provided to incorporate flexible quantile regressions with an efficient fairness adjustment post-processing algorithm.
We demonstrate the superior empirical performance of this approach on several benchmark datasets. 
Our results show the model's ability to uncover the mechanism underlying the fairness-accuracy trade-off in a wide range of societal and medical applications. 

\end{abstract}

\section{Introduction}\label{sec:intro}

{
We are increasingly leaning on machine learning systems to tackle human problems. 
A primary objective is to develop intelligent algorithms that can automatically produce accurate decisions which also enjoy equitable properties, as unintended social bias has been identified as a rising concern in various fields \cite{czarnowska2021quantifying,gardeazabal_ugidos_2005,ding2022word}.

As a means of providing quantitative measures of fairness, a number of metrics have been proposed. These metrics can be categorized into three broad categories: group fairness \cite{barocas2017fairness}, individual fairness \cite{Kusner_Loftus_Russell_Silva}, and causality-based fairness \cite{Plecko_Meinshausen}. In contrast to causality-based fairness that requires domain knowledge to develop a fair causal structure and individual fairness that seeks equality only between similar individuals, group fairness does not require any prior knowledge and seeks equality for groups as a whole \cite{castelnovo2022clarification}. 
Among the metrics defined for  group fairness such as equalized odds \cite{Cho_Hwang_Suh,Oneto_Donini_Pontil_2019} and predictive rate parity \cite{chouldechova2017fair}, 
demographic parity (DP) is generic since it does not allow prediction results in aggregate to depend on sensitive attributes \cite{Agarwal_Dudik_Wu_2019, holzer2006affirmative,chzhen_denis_hebiri_oneto_pontil_2020,Yang_Lafferty_Pollard_2019}. 
In particular, an algorithm is said to satisfy DP if its prediction is independent of any given sensitive attribute \cite{Agarwal_Dudik_Wu_2019}. 

 There have been a number of studies on algorithmic fairness concerning DP \cite{ Agarwal_Dudik_Wu_2019,Chzhen_Schreuder_2022,chzhen_denis_hebiri_oneto_pontil_2020,holzer2006affirmative,plevcko2021fairadapt,Yang_Lafferty_Pollard_2019}. 
 In the context regression analysis, much attention have been paid on conditional mean inferences \cite{Agarwal_Dudik_Wu_2019,Chzhen_Schreuder_2022,chzhen_denis_hebiri_oneto_pontil_2020,plevcko2021fairadapt}, few are concerned with conditional quantiles \cite{williamson2019fairness,Yang_Lafferty_Pollard_2019}. 
 As real-world data often exhibit heterogeneity, contain extreme outliers, or do not meet satisfactory distributional assumptions, like Gaussianity, a fairness discussion on conditional quantiles may be more rational and essential since they are able to provide a more complete understanding of the dependence structure between response and explanatory variables \cite{Yang_Lafferty_Pollard_2019}, as well as better accommodate asymmetry and extreme tail behavior \cite{xiao2009conditional}. It should also be noted that bias or unfairness that arises in mean regression may also be propagated through quantile regression, therefore it must be properly dealt with separately: a graphic demonstration can be found in Figure \ref{fig:result:illustration}. More intuitively, we may take an example from a Spanish labor market study \citep{garcia_hernandez_2001,gardeazabal_ugidos_2005}. The study found that in Spain, also in line with other countries, the mean wage gap between men and women is quite substantial: on average, women earn around 70 percent of what men earn. While wage gaps are not uniform across all pay scales, they are greater at higher quantiles than at lower quantiles. As biases and disparities at different quantiles tend to be overshadowed by the mean behavior of the entire population, we propose a novel framework to seek fair predictions at different quantiles. It uses optimal transport techniques \cite{Agueh_Carlier_2011,chzhen_denis_hebiri_oneto_pontil_2020} by transforming \emph{bias-affected} distributions into an \emph{only-fair} Wasserstein-2 barycenter through a kernel-based functional synchronization method \cite{cheng1997unified,Zhang_Muller_2011}, in order to provide fair quantile estimators.

\begin{minipage}[!htbp]{\textwidth}
  \centering
  \captionsetup{type=figure}
    \includegraphics[width=0.6\textwidth]{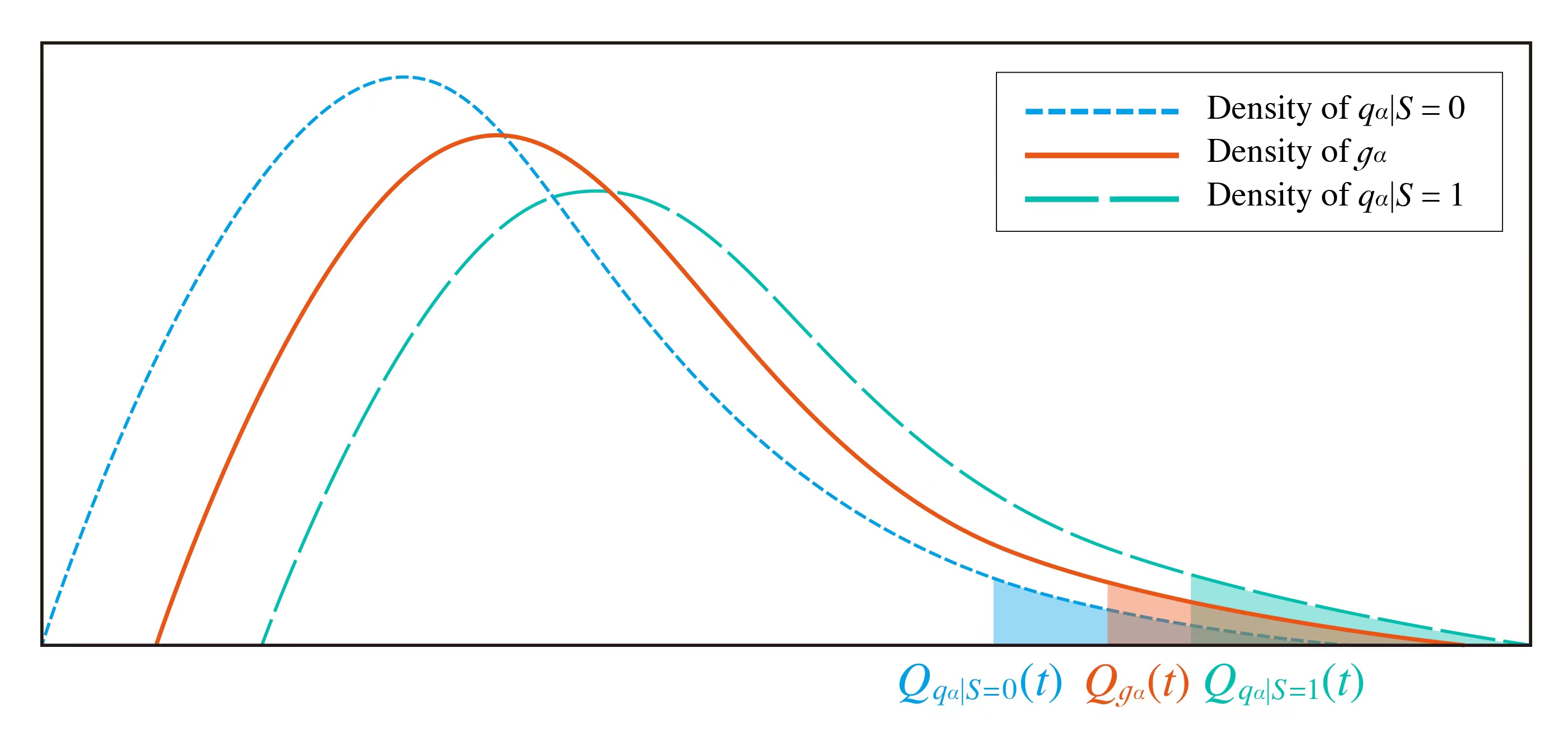}
    \captionof{figure}{An illustration of quantile fairness: for a skewed and heteroscedastic quantile estimation $\{q_{\alpha,i}\}_{i=1}^N$ affected by the sensitive attribute $S\in \{0,1\}$, for example, the higher quantile of the salary distribution, the optimal fair quantile prediction $Q_{g_\alpha}(t), t\in(0,1)$ is derived through a convex combination of the conditional quantile functions of $Q_{q_\alpha|S=0}$ and $Q_{q_\alpha|S=1}$. 
    }
\label{fig:result:illustration}
\end{minipage}

Since quantile fairness poses a number of theoretical challenges, no previous literature has been able to provide any inference results such as prediction intervals concerning quantile fairness. It is imperative to keep in mind that fairness is only one of two legs of the primary goal of modern machine learning algorithms, the other being accuracy. Building a reliable prediction with valid confidence is a significant challenge that is encountered by many machine learning algorithms \cite{zeni2020conformal}. Towards this end, we propose the conformalized fair quantile prediction (CFQP) inspired by the works of \citet{Romano_Barber_Sabatti_Candes_2019,Romano_Patterson_Candes_2019}. Our analysis demonstrates, both mathematically and experimentally, that CFQR provides finite sample, distribution-free validity, DP fairness for different quantiles, and precise control of the miscoverage rate, regardless of the underlying quantile algorithm.

\textbf{Contributions and Outlines.}
In this paper, we propose a new quantile based method with valid inference that enhances both accuracy and fairness while maintaining a balance between the two. It is a novel framework that allows an exact control of prediction miscoverages while ensuring quantile fairness simultaneously. The main contributions are summarized as follows:

\begin{enumerate}[i.]
    \item 
    We successfully transform the problem of searching quantiles under DP fairness to the construction of  multi-marginal Wasserstein-2 barycenters via the optimal transport theory \cite{Agueh_Carlier_2011,chzhen_denis_hebiri_oneto_pontil_2020,Gouic_Loubes_Rigollet_2020}. We incorporate a novel kernel smoothing step into the preceding method, which is particularly advantageous for subgroups whose sample sizes are too small to obtain reliable quantile function estimations. 
    
\item 
  In Section \ref{sec:CFP}, we propose a conformalized fair quantile regression prediction interval (CFQP) inspired by the works of \citet{Romano_Barber_Sabatti_Candes_2019,Romano_Patterson_Candes_2019}. It is mathematically proved to achieve a distribution-free validity, demographic parity on different quantiles, and an exact control of miscoverage rates, regardless of the quantile algorithm used.
  The theoretical validity of prediction interval constructed by CFQP and exact DP of the fair quantile estimators are given in Section \ref{sec:theory} and the supplement.

\item 
The experimental results presented in Section \ref{sec:experiments} include a numerical comparison of the proposed CFQP and fair quantile estimation with both state-of-the-art conformal and fairness-oriented methods. By reducing the discriminatory bias dramatically, our method outperforms the state-of-the-art methods while maintaining reasonable short interval lengths.

\end{enumerate}

\textbf{Related works.}
Existing approaches for building a fair mean regression broadly fall into three classes: pre-processing, in-processing and post-processing. In particular, preprocessing methods focus on transforming the data to remove any unwanted bias \cite{calmon2017optimized,Plecko_Meinshausen,zemel2013learning}; in-processing 
methods aim to build in fairness constraints into the training step \cite{Agarwal_Dudik_Wu_2019,berk2017convex,Kusner_Loftus_Russell_Silva,Oneto_Donini_Pontil_2019}; post-processing methods target to modify the trained predictor \cite{Chzhen_Schreuder_2022,chzhen_denis_hebiri_oneto_pontil_2020,milli2019social}.
As few previous works have focused on the quantile fairness of and fair prediction interval, the most related are \citet{Yang_Lafferty_Pollard_2019}, where a different fairness measure was used. While \citet{Agarwal_Dudik_Wu_2019} mentioned that their reduction-based approach can be adapted into quantile regression, \citet{williamson2019fairness} brought forward a novel conditional variance at risk fairness measure aiming to control the largest subgroup risk. For interval fairness measure, the approach by \citet{Romano_Barber_Sabatti_Candes_2019} achieved equalized coverage among groups without fairness on interval endpoints. 
Methodologically, integrating algorithmic fairness with Wasserstein distance based barycenter problem has been studied in \cite{Agueh_Carlier_2011,Chzhen_Schreuder_2022,chzhen_denis_hebiri_oneto_pontil_2020,Gouic_Loubes_Rigollet_2020,Jiang_Pacchiano_Stepleton_Jiang_Chiappa_2020}. Both in-processing \cite{Agarwal_Dudik_Wu_2019,Jiang_Pacchiano_Stepleton_Jiang_Chiappa_2020} and post-processing \cite{Chzhen_Schreuder_2022,chzhen_denis_hebiri_oneto_pontil_2020} methods were proposed to solve classification and mean regression problems. As a post-processing method, our work is distinct from above-mentioned methods by constructing the DP-fairness for each population quantile, and generating a fair prediction interval accordingly. 
}

\textbf{Notations.}
We denote by $[K]$ the set $\{1, \ldots, K\}$ for arbitrary integer $K$. $|\mathcal{S}|$ represents the cardinality for a finite set $\mathcal{S}$.  ${E}$ and ${P}$ represent the 
expectation and probability and $\mathbbm{1}\{\cdot\}$ is the indicator function. 
Let $\{ Z_n \}_{n=1}^{\infty}$ be a sequence of random variables, and $\{k_n\}_{n=1}^\infty$ be a sequence of positive numbers,
we say that $Z_{n}=O_p(k_{n})$, if
$\lim _{T \rightarrow \infty} \limsup _{n \rightarrow \infty} {P}(|Z_{n}|>T k_{n})=0$, then $Z_{n}/k_{n} =O_{p}(1)$. To denote the equality in distribution of two random variables $A$ and $B$, we write $A\overset{d}{=}B$. 

\section{Problem statement }\label{sec:prob_stat}

Consider the regression problem where a “sensitive characteristic” $S$ is available, which by the U.S. law \cite{Gouic_Loubes_Rigollet_2020,Romano_Barber_Sabatti_Candes_2019} can be enumerated as sex, race, age, disability, etc.
We observe the triplets $\left({X}_{1}, S_{1}, Y_{1}\right), \dots,\left({X}_{n}, S_{n}, Y_{n}\right)$, denote $({X}_{i}, S_{i}, Y_{i})$ by $Z_i$, $i=1,\dots, n$ and $Z_i$ is a random variable in $ \mathbb{R}^{p} \times[K] \times \mathbb{R}$. The aim is to predict the unknown value of $Y_{n+1}$ at a test point $X_{n+1},S_{n+1}$. 
Let ${P}$ be the joint distribution of $Z$, we assume that all the samples $\left\{Z_i\right\}_{i=1}^{n+1}$ are drawn exchangeable, where i.i.d. is a special case.

Our goal is to construct a marginal distribution-free prediction band $C\left(X_{n+1},S_{n+1}\right) \subseteq \mathbb{R}$ that is likely to cover the unknown response $Y_{n+1}$ with finite-sample (nonasymptotic) validity. Formally, given a desired miscoverage rate $\alpha$, the predicted interval satisfies
\begin{equation}\label{sec1:eq:quantl_coverage}
  {P}\left\{Y_{n+1} \in C\left(X_{n+1},S_{n+1}\right)\right\} \geq 1-\alpha  
\end{equation}
for any joint distribution ${P}$ and any sample size $n$, while the left and right endpoint of $C\left(X_{n+1},S_{n+1}\right)$ satisfies the fairness constraint of Demographic Parity concerning the sensitive variable $S$. 

\textbf{Demographic Parity.} 
We introduce the quantitative definition of DP in fair regression and connect the DP-fairness with a quantile regressor $q_\alpha$. The result that $q_\alpha$ can be projected to the fair counterparts using optimal transport will be invoked later. 

Given a fixed quantile level $\alpha$ (it may refer to ${\alpha_{\mathrm{lo}}}$ or ${\alpha_{\mathrm{hi}}}$ indicating the upper and lower quantile estimates for the prediction band $C(X_{n+1},S_{n+1})$). Let $q_{\alpha}: \mathbb{R}^{p} \times [K] \rightarrow \mathbb{R}$ represent an arbitrary conditional quantile predictor.  Denote by $\nu_{q_{\alpha} \mid s}$ the distribution of $(q_{\alpha}(X, S) \mid S=s)$, the Cumulative Distribution Function (CDF) of $\nu_{q_{\alpha} \mid s}$ is given by
\begin{equation}
    F_{\nu_{q_{\alpha} \mid s}}(t)={P}(q_{\alpha}(X, S) \leq t \mid S=s).
\end{equation}
The quantile function $Q_{\nu_{q_{\alpha} \mid s}}=F^{-1}_{\nu_{q_{\alpha} \mid s}}:[0,1] \rightarrow \mathbb{R}$ ,namely, the generalized inverse of $F_{\nu_{q_{\alpha} \mid s}}$, can thus be defined as for all levels $t \in (0, 1]$,
\begin{equation}\label{sec:fair:def_quantile}
    Q_{\nu_{q_{\alpha} \mid s}}(t) = \inf \{y \in \mathbb{R} : F_{\nu_{q_{\alpha} \mid s}}(y)\ge t\} \ \text{with}\  Q_{\nu_{q_{\alpha} \mid s}}(0) = Q_{\nu_{q_{\alpha} \mid s}}(0+).
\end{equation}
To simplify the notations, we will write $F_{q_{\alpha} \mid s}$ and $Q_{q_{\alpha} \mid s}$ instead of $F_{\nu_{q_{\alpha} \mid s}}$ and $Q_{\nu_{q_{\alpha} \mid s}}$ respectively, for any prediction rule $q_{\alpha}$.

In the following, we introduce the definition of Demographic 
Parity (DP), which is most commonly used in the context of fairness research  \cite{ Agarwal_Dudik_Wu_2019,Chzhen_Schreuder_2022,chzhen_denis_hebiri_oneto_pontil_2020,holzer2006affirmative,Oneto_Donini_Pontil_2019}.

\begin{definition}[Demographic Parity]\label{def:Demo_Par}
An arbitrary prediction $g: \mathbb{R}^{d} \times [K] \rightarrow \mathbb{R}$ satisfies demographic parity under a distribution ${P}$ over $(X, S, Y)$, if $g(X,S)$ is statistically independent of the sensitive attribute $S$. Formally, for every $s, s^{\prime} \in [K]$,
$$
\sup _{t \in \mathbb{R}}\left|{P}(g(X, S) \leq t \mid S=s)-{P}\left(g(X, S) \leq t \mid S=s^{\prime}\right)\right|=0.
$$
\end{definition}
Demographic Parity (DP) requires the predictions to be independent of the sensitive attribute, and it demands the Kolmogorov-Smirnov distance \cite{lehmann2005testing} (the difference between CDFs measured in the $l_\infty$ norm) between $\nu_{g \mid s}$ and $\nu_{g \mid s^{\prime}}$ to vanish for all categories $s, s^{\prime}$. 

\section{Quantile Regression and Conformal Prediction}\label{sec:qt_cqr}

In this section, we recall the CQR approach for finite sample, distribution-free prediction interval inference. Quantile regression was proposed by \citet{koenker1978regression} to estimate the $\alpha$-th quantile of the conditional distribution of $Y$ given $\tilde{X}:=(X, S)$ for some quantile level $\alpha\in(0,1)$, since then it has become more pervasive with various applications, such as providing prediction intervals, detecting outliers, or perceiving the entire distribution \cite{Meinshausen,han2019general}. 
Denote the conditional cumulative distribution of $Y$ given $\tilde{X}$ by $F(y\mid \tilde{X}=\tilde{x}):={P}\{Y\le y\mid \tilde{X}=\tilde{x}\}$. The $\alpha$-th conditional quantile prediction is defined as 
$q_\alpha(\tilde{x}):=\inf\{y\in \mathbb{R}:F(y\mid \tilde{X}=\tilde{x})\ge \alpha\}.$
Quantile regression can be cast as an optimization problem \cite{Meinshausen,yu2019sparse,Pietrosanu_Gao_Kong_Jiang_Niu_2021,wang2019wavelet,yu2016partial}, by minimizing the expected check loss function $E(\rho_\alpha)=E[\rho_\alpha(y,q)|\tilde{X}=\tilde{x}]$, where
\begin{equation}\label{sec:quant:eq:quant_obj}
\rho_\alpha(y,q_{\alpha}(\tilde{x}))=\left\{\begin{aligned}
     \alpha |y-q_{\alpha}(\tilde{x})|\quad  &\text{if} \ y\ge q_{\alpha}(\tilde{x}),\\
     (1-\alpha)|y-q_{\alpha}(\tilde{x})|\quad  &\text{if}\   y<q_{\alpha}(\tilde{x}).
     \end{aligned}
  \right.
 \end{equation}

Quantile regression offers a principled way of judging the reliability of predictions by building a prediction interval for the new observation $(\tilde{X}_{n+1}, Y_{n+1})$. In contrast to asymptopia, 
\citet{Romano_Barber_Sabatti_Candes_2019,Romano_Patterson_Candes_2019} brought forward the conformalized quantile regression (CQR) by combining the merits of robust quantile regression with conformal prediction, thus finite sample validity in Eq. \eqref{sec1:eq:quantl_coverage} is guaranteed. Inspired by the split conformal method, a split CQR likewise starts with splitting the data into a proper training set and a calibration set, indexed by $\mathcal{I}_{1}, \mathcal{I}_{2}$ respectively. Given any quantile regression algorithm $\mathcal{Q}$, we then fit two conditional quantile functions $\hat{q}_{\alpha_{\mathrm{lo}}}$ and $\hat{q}_{\alpha_{\mathrm{hi}}}$ on the proper training set:
$
\left\{\hat{q}_{\alpha_{\mathrm{lo}}}, \hat{q}_{\alpha_{\mathrm{hi}}}\right\} \leftarrow \mathcal{Q}\left(\left\{\left(\tilde{X}_{i}, Y_{i}\right): i \in \mathcal{I}_{1}\right\}\right).
$
The conformity scores are calculated to quantify the error made by the plug-in prediction interval $\hat{C}(\tilde{x})=[\hat{q}_{\alpha_{\mathrm{lo}}}(\tilde{x}), \hat{q}_{\alpha_{\mathrm{hi}}}(\tilde{x})]$. We evaluate the scores on the calibration set as
$
E_{k}:=\max \left\{\hat{q}_{\alpha_{\mathrm{lo}}}(\tilde{X}_{k})-Y_{k}, Y_{k}-\hat{q}_{\alpha_{\mathrm{hi}}}(\tilde{X}_{k})\right\}
$
for each $k \in \mathcal{I}_{2}$,
where both undercoverage and overcoverage of the interval are taken into consideration \cite{Romano_Patterson_Candes_2019}. 
Given a new input data $\tilde{X}_{n+1}$, we construct the prediction interval for $Y_{n+1}$ as
$
C\left(\tilde{X}_{n+1}\right)=\left[\hat{q}_{\alpha_{\mathrm{lo}}}\left(\tilde{X}_{n+1}\right)-Q_{1-\alpha}\left(E, \mathcal{I}_{2}\right), \hat{q}_{\alpha_{\mathrm{hi}}}\left(\tilde{X}_{n+1}\right)+Q_{1-\alpha}\left(E, \mathcal{I}_{2}\right)\right],
$
where $Q_{1-\alpha}(E, \mathcal{I}_2) := (1-\alpha)(1 + 1/|\mathcal{I}_2|)$-th empirical quantile of $\{E_k : k\in \mathcal{I}_2\}$ conformalizes the plug-in prediction interval. Note that the constructed interval $C(\tilde{X}_{n+1})$ could be highly influenced by the sensitive variable $S$. 

\section{Conformal fair quantile prediction (CFQP)}\label{sec:CFP}

We formally describe our proposed conformal fair prediction (CFQP) framework for constructing DP fairness constrained prediction intervals in this section. A kernel smoothing quantile function is introduced during the functional synchronization, which can improve the estimation when some subgroups are too small to give reliable sample quantile function estimations. 

\begin{definition}[Wasserstein-2 distance]
 Let $\mu$ and $\nu$ be two univariate probability measures with finite second moments. The squared Wasserstein-2 distance between $\mu$ and $\nu$ is defined as
$$
\mathcal{W}_{2}^{2}(\mu, \nu)=\inf\left\{\int_{\mathbb{R}\times\mathbb{R}}|x-y|^{2} d \gamma(x, y), \gamma \in \Gamma_{\mu, \nu}\right\}
$$
where $\Gamma_{\mu, \nu}$ is the set of probability measures (couplings) on $\mathbb{R} \times \mathbb{R}$   having $\mu$ and $\nu$ as marginals.
\end{definition}

\begin{proposition}[Fair optimal prediction \cite{chzhen_denis_hebiri_oneto_pontil_2020}]\label{prop1: fair_opt_pred}
Assume, for each $s \in [K]$, that the univariate measure $\nu_{q_{\alpha} \mid s}$ has a density and let $p_{s}={P}(S=s)$. Then,
\begin{equation}
   \min _{g_{\alpha} \text { is fair }} {E}\left(q_{\alpha}(X, S)-g_{\alpha}(X, S)\right)^{2}=\min _{\nu} \sum_{s \in [K]} p_{s} \mathcal{W}_{2}^{2}\left(\nu_{q_{\alpha} \mid s}, \nu\right) . 
\end{equation}
Moreover, if $g_{\alpha}$ and $\nu$ solve the l.h.s. and the r.h.s. problems respectively, then $\nu=\nu_{g_{\alpha}}$ and specifically, 
\begin{equation}\label{sec:fair:eq:g_fair}
g_{\alpha}(x,s)=\sum_{s^{\prime} \in [K]} p_{s^{\prime}} Q_{q_{\alpha} \mid s^{\prime}} \circ F_{q_{\alpha} \mid s}\circ q_{\alpha}(x, s) .
\end{equation}
\end{proposition}

Proposition \ref{prop1: fair_opt_pred} implies that the optimal fair quantile predictor for an input $(x, s)$ is obtained by a nonlinear transformation of the vector $[q_{\alpha}(x, s)]_{s=1}^{K}$ linking to a Wasserstein barycenter problem\cite{Agueh_Carlier_2011,chzhen_denis_hebiri_oneto_pontil_2020}.
The explicit closed form solution comes from \cite{Agueh_Carlier_2011,chzhen_denis_hebiri_oneto_pontil_2020,Gallon_Loubes_Maza_2013}, which relies on the classical characterization of optimal coupling in one dimension of the Wasserstein-2 distance. A rigorous proof is given in \cite{chzhen_denis_hebiri_oneto_pontil_2020,Jiang_Pacchiano_Stepleton_Jiang_Chiappa_2020}.
It shows that a minimizer $g_{\alpha}$ of the $L_{2}$-risk can be used to construct $\nu$ and vice-versa, given $\nu$, there is a explicit expression Eq. \eqref{sec:fair:eq:g_fair} for the multi-marginal Wasserstein barycenter \cite{Agueh_Carlier_2011}.

{ 
First, we provide a sketch of the CFQP approach. 
We start with splitting the whole training data into a proper training set $\mathcal{I}_1$ and a calibration set $\mathcal{I}_2$, 
then fit an arbitrary quantile regression algorithm $\mathcal{Q}$ on $\mathcal{I}_1$,
$
\left\{\hat{q}_{\alpha_{\mathrm{lo}}}, \hat{q}_{\alpha_{\mathrm{hi}}}\right\} \leftarrow \mathcal{Q}\left(\left\{\left(\tilde{X}_{i}, Y_{i}\right): i \in \mathcal{I}_{1}\right\}\right).
$
We apply the fitted quantile algorithm $\mathcal{Q}$ on the calibration set $\mathcal{I}_2$ to obtain the predicted $\{\hat{q}_{\alpha_{\mathrm{lo}}}(\tilde{X}_i), \hat{q}_{\alpha_{\mathrm{hi}}}(\tilde{X}_i)\}_{i\in\mathcal{I}_2}$. Since the quantile estimates for $\mathcal{I}_2$ will be used for conformalization, it is essential to transform them into fair ones, i.e. $\hat{g}_{\alpha,i}, \forall i\in\mathcal{I}_2$ (Eq. (9)), through Algorithm \ref{alg:quantile_normlz}. Finally, for a test point $\tilde{X}_{n+1}$, we will predict two quantile estimates $\hat{q}_{\alpha}(x,s)$ affected by the sensitive variable $S$ by $\mathcal{Q}$, then apply the functional synchronization (details in Algorithm \ref{alg:quantile_normlz}) and calibration (Algorithm \ref{alg:cfp}) steps in turn to generate the fair constraint prediction interval $ C(\tilde{X}_{n+1})$ for $Y_{n+1}$.
}

Next, we explicate in detail how to remove the effect of the sensitive variable for the predicted quantile estimates.
By Proposition \ref{prop1: fair_opt_pred}, the optimal fair quantiles take the form of Eq. \eqref{sec:fair:eq:g_fair}. Therefore, we propose an empirical optimal fair quantile estimator $\hat{g}_\alpha$ that relies on the plug-in principle. In particular, Eq.\eqref{sec:fair:eq:g_fair} indicates that for each quantile level $\alpha$ and each category $s \in [K]$, we only need estimators for the regression function $\hat{q}_{\alpha}$, the proportion $\hat{p}_{s}$,  the cumulative distribution function $F_{\hat{q}_{\alpha} \mid s}$ and the quantile function $Q_{\hat{q}_{\alpha} \mid s}$.

Note that we can empirically estimate the CDF and quantile function for each sensitive group in the calibration set $\mathcal{I}_2$ separately. 
Hence for each quantile level $\alpha$,
let $N_s:=|\mathcal{I}_2^s|$, and the quantile estimators $(\hat{q}_1^s,\hat{q}_2^s,\dots,\hat{q}_{N_s}^s)$\footnote{$\hat{q}_i^s$ depends on the quantile level $\alpha$, we suppress $\alpha$ for notational simplification.} are  calculated through the fitted quantile regression $\mathcal{Q}$ 
with training set $\mathcal{I}_1$.  
We define the augmented random variable for each data point in $\mathcal{I}_2$,
$$
\tilde{q}_{i}^{s}:=\hat{q}_{i}^{s}+U_i^{s}([-\sigma,\sigma]) \quad \forall i\in \mathcal{I}_2^s,  s\in [K],
$$

where $U_i^{s}$ are i.i.d. random variables, uniformly distributed on $[-\sigma, \sigma]$ for some small positive $\sigma$, and independent from all the previously introduced random variables. It serves as a smoothing random variable, for the random variables $\tilde{q}_{i}^{s}$ are i.i.d. continuous for any ${P}(Y|\tilde{X})$ and $\mathcal{Q}$. Otherwise, the original $\hat{q}_i^s$ might have atoms resulting in a non-zero probability to observe ties in the group $\{\hat{q}_i^s\}$ for $s=1,\dots,K$. This trick, also called jittering \cite{chambers2018graphical,chzhen_denis_hebiri_oneto_pontil_2020} is often used for data visualization for tie-breaking.
Using the above quantities, we build the CDF and quantile function estimators for each subgroup $s^{\prime}\in[K]$ as follows,
\begin{align}
&\label{sec:cfr:eq:F_hat_q}\hat{F}_{q_{\alpha}|s^{\prime}}(t)= N_s^{-1} \sum_{i=1}^{N_s} \mathbbm{1}\left\{\tilde{q}_{i}^{s^\prime} \leq t\right\}, \quad\text{for all } t \in \mathbb{R},\\
&\label{eq:sec:cfr:F_hat_2inv}\hat{Q}_{2, q_{\alpha}|s^{\prime}}(t)=\int_0^1 \hat{F}^{-1}_{q_{\alpha}|s^{\prime}}(v)K_h(t-v)dv, \quad t\in (0,1).
\end{align}
The smoothed kernel estimator Eq.\eqref{eq:sec:cfr:F_hat_2inv} was firstly proposed by \citet{cheng1997unified}, where 
$K_h(\cdot)=K(\cdot/h)/h$ is a kernel function chosen as a probability density function that is symmetric around zero with bandwidth parameter $h > 0$.

If the quantile functions ${Q}_{2, q_{\alpha}|s^{\prime}}$ is differentiable, the derivative $Q^{\prime}_{s^{\prime}}(t):={Q}_{2, q_{\alpha}|s^{\prime}}^{\prime}(t) $ for $t \in(0,1)$ is the quantile density function \cite{cheng1997unified,Zhang_Muller_2011}. We hereby give an estimation bound for Eq. \eqref{eq:sec:cfr:F_hat_2inv} using kernel smoothing. For this purpose, we invoke the conditions (A1) - (A3)  that are needed for deducing the following proposition. They can also be found from  \cite{Zhang_Muller_2011} and are included in the supplementary material.

\begin{proposition}\label{sec:cfr:prop:smooth_quantile_est}
 Under conditions (A1), (A2), and (A3), we have
$$
\sup _{s^{\prime}} \sup _{t \in[0,1]}\left|\hat{Q}_{2, q_{\alpha}|s^{\prime}}(t)-{Q}_{q_{\alpha}|s^{\prime}}(t)\right|=O_{p}\left(N^{-\nicefrac{1}{2}}\right), \quad s^{\prime}=1, \ldots, K.
$$
\end{proposition}

The motivation for including a smoothing step is twofold: 
First, smoothing the quantile function eliminates the troublesomeness in defining arbitrary quantiles from the empirical one when the sample sizes of subgroups are small. 
Second, the proposed kernel smoothing improves second-order efficiency by alleviating the relative deficiency \cite{falk1984relative,Zhang_Muller_2011}. 
\begin{remark}
One can utilize various kernels such as the Gaussian or Epanechnikov kernel with adaptive bandwidth for better practical performance. Other smoothing methods such as splines or local linear fitting can likewise be applied with equal effectiveness. 
\end{remark}
Consequently, for each quantile level $\alpha$, the functional synchronized quantile estimator is 
\begin{equation}\label{eq:sec:cfr:ghat_est}
\hat{g}_{\alpha,i}=\sum_{s^{\prime}=1}^{K} \hat{p}_{s^{\prime}} \hat{Q}_{2, q_{\alpha}|s^{\prime}}\circ \hat{F}_{q_{\alpha}|s} \circ \tilde{q}_i^s, \quad  \forall i\in\mathcal{I}_2.
\end{equation}
The proposed estimator can be deemed as the empirical counterpart with additional randomization of the explicit fair optimal formula Eq.\eqref{sec:fair:eq:g_fair}.

To conformalize the adjusted fair quantiles Eq \eqref{eq:sec:cfr:ghat_est},  we need to compute the conformity scores $E_i$ for each $i \in \mathcal{I}_2$ that quantify the error made by the plug-in fair prediction interval $\hat{C}^g(\tilde{x})=[\hat{g}_{\alpha_{\mathrm{lo}}}(\tilde{x}),\hat{g}_{\alpha_{\mathrm{hi}}}(\tilde{x})]$. The scores are evaluated on the calibration set as 
\begin{equation}
    E_i:=\max\{\hat{g}_{\alpha_{\mathrm{lo}},i}-Y_i,Y_i-\hat{g}_{\alpha_{\mathrm{hi}},i}\}.
\end{equation}
At the last stage, for a new data point $\tilde{X}_{n+1}=(x,s)$, and $\alpha\in\{\alpha_{\mathrm{lo}},\alpha_{\mathrm{hi}}\}$, by defining 
$$
    \tilde{q}_{1,i}^{s}=\hat{q}_{i}^{s} +U_i^{s}([-\sigma,\sigma]) \quad \forall i\in \mathcal{I}_1^s \quad \text{ and } \quad \tilde{q}_{\alpha}(x,s)=\hat{q}_{\alpha}(x,s)+U([-\sigma,\sigma]).
$$

We use the empirical CDF of training set \footnote{Still, $\hat{q}_{i}^{s}$ depends on quantile level $\alpha$.}
\begin{equation}
    \textstyle{\hat{F}_{1, q_{\alpha}|s}(t):=\frac{1}{|\mathcal{I}_1^s|+1}\left(\sum_{i=1}^{|\mathcal{I}_1^s|} \mathbbm{1}\left\{\tilde{q}_{1,i}^{s}< t\right\}+U([0,1])\left(1+\sum_{i=1}^{|\mathcal{I}_1^s|} \mathbbm{1}\left\{\tilde{q}_{1,i}^{s}=t\right\}\right)\right)}
\end{equation}
to estimate the location $\hat{F}_{1, q_{\alpha}|s}\circ \tilde{q}_{\alpha}(x,s)$. Thus the fair quantile estimator is built as follows 
\begin{equation}\label{sec:cfr:eq:g_hat_test}
  \hat{g}_{\alpha}(x,s)=\sum_{s^{\prime}=1}^{K} \hat{p}_{s^{\prime}} \hat{Q}_{2, q_{\alpha}|s^{\prime}} \circ \hat{F}_{1, q_{\alpha}|s} \circ \tilde{q}_{\alpha}(x,s),\forall \alpha\in\{\alpha_{\mathrm{lo}},\alpha_{\mathrm{hi}}\}.  
\end{equation}

The fair prediction interval for $Y_{n+1}$ is constructed as
\begin{equation}
    C(\tilde{X}_{n+1}) = [\hat{g}_{\alpha_{\mathrm{lo}}}(x,s)- Q_{1-\alpha}(E, \mathcal{I}_2),\hat{g}_{\alpha_{\mathrm{hi}}}(x,s)+ Q_{1-\alpha}(E, \mathcal{I}_2)],
\end{equation}

where $Q_{1-\alpha}(E, \mathcal{I}_2) := (1-\alpha)(1 + 1/|\mathcal{I}_2|)$-th empirical quantile of $\{E_i : i\in \mathcal{I}_2\}$ will adjust the plug-in fair prediction interval.
We present the pseudo-codes of CFQP as well as the construction of $\hat{g}_{\alpha}$ for Eq. \ref{eq:sec:cfr:ghat_est} in Algorithm \ref{alg:cfp}, \ref{alg:quantile_normlz} respectively. 

\setcounter{algorithm}{0}
\begin{algorithm}[htb]\small
\caption{\small Split Conformal Fair Prediction (CFQP)}
\label{alg:cfp}
\begin{algorithmic}[1] 
\renewcommand{\algorithmicrequire}{{{\bf Input}:}}
\REQUIRE 
$\mathcal{D}=\{(X_i,S_i, Y_i)\}_{i=1}^n$; miscoverage level $ \alpha \in (0, 1)$; quantile regression algorithm $\mathcal{Q}$.

\STATE Randomly split $[n]$ into disjoint proper training and calibration indices $\mathcal{I}_1$,  $\mathcal{I}_2$.

\STATE Fit two conditional quantile functions on the training set $\{\hat{q}_{\alpha_{\mathrm{lo}}},\hat{q}_{\alpha_{\mathrm{hi}}}\}\leftarrow\mathcal{Q}(\{(X_i,S_i,Y_i), i\in\mathcal{I}_1\})$.

\STATE Call functional Synchronization (Algorithm \ref{alg:quantile_normlz}) to calculate $\{\hat{g}_{\alpha_{\mathrm{lo}}},\hat{g}_{\alpha_{\mathrm{hi}}}\}$ for each $i\in \mathcal{I}_2$.

\STATE Compute $E_i\leftarrow\max \left\{\hat{g}_{\alpha_{\mathrm{lo}}}\left(X_{i}\right)-Y_{i}, Y_{i}-\hat{g}_{\alpha_{\mathrm{hi}}}\left(X_{i}\right)\right\}$ for $ \forall \ i \in \mathcal{I}_2$.

\STATE Compute $Q_{1-\alpha}(E,\mathcal{I}_2) \leftarrow (1 -\alpha)(1+ \nicefrac{1}{|\mathcal{I}_2|})$-th empirical quantile of $\{E_i : i \in \mathcal{I}_2\}$.

\STATE For a new test point $(x,s)$, compute $\{\hat{g}_{\alpha_{\mathrm{lo}}}(x,s),\hat{g}_{\alpha_{\mathrm{hi}}}(x,s) \}$ through Algorithm \ref{alg:quantile_normlz}

\renewcommand{\algorithmicrequire}{{{\bf Output}:}} 
\REQUIRE
Fair prediction interval $C(x,s)=[\hat{g}_{\alpha_{\mathrm{lo}}}(x,s)-Q_{1-\alpha}(E,\mathcal{I}_2),\hat{g}_{\alpha_{\mathrm{hi}}}(x,s)+Q_{1-\alpha}(E,\mathcal{I}_2)]$ for $(X_{n+1},S_{n+1}) = (x, s)$.
\end{algorithmic}
\end{algorithm}

\setcounter{algorithm}{1}
\begin{algorithm}[htb]\small
\caption{\small Functional Synchronization}
\label{alg:quantile_normlz}
\begin{algorithmic}[1] 
\renewcommand{\algorithmicrequire}{{{\bf Input}:}}
\REQUIRE 
Calibration set $\{(X_i,S_i)\}_{i\in\mathcal{I}_2}$ or new point $(x,s)$; base quantile estimator $\mathcal{Q}$;\\ slack parameter $\sigma$; 
training set $\{(X_i,S_i)\}_{i\in\mathcal{I}_1}$;
\IF{Calibration set $\{(X_i,S_i)\}_{i\in\mathcal{I}_2}$} \FOR{$\alpha\in\{\alpha_{\mathrm{lo}},\alpha_{\mathrm{hi}}\}$}
\STATE $\{\tilde{q}_{\alpha}(X_i,S_i)\}\leftarrow \{q_{\alpha}(X_i,S_i)+U([-\sigma,\sigma])\}_{i\in \mathcal{I}_2}$   \hfill{$\triangleright\  U([-\sigma,\sigma])$ are used for tie-breaking}

\FOR{$s^\prime\in [K]$}
\STATE 
Compute $\hat{F}_{q_{\alpha}|s^{\prime}}(t)$, and $\hat{F}^{-1}_{2, q_{\alpha}|s^{\prime}}(t)$ by Eq. \eqref{sec:cfr:eq:F_hat_q} and \eqref{eq:sec:cfr:F_hat_2inv}.

\STATE
Obtain $\hat{g}_{\alpha}(X_i, S_i)\leftarrow\sum_{s^{\prime}=1}^{K} \hat{p}_{s^{\prime}} \hat{F}^{-1}_{2, q_{\alpha}|s^{\prime}} \circ \hat{F}_{q_{\alpha}|s^{\prime}} \circ \tilde{q}_{\alpha}(X_i, S_i) , \ \forall i\in \mathcal{I}_2$
\ENDFOR
\ENDFOR

\ELSIF {New test point $(x,s)$}
\FOR{$\alpha\in\{\alpha_{\mathrm{lo}},\alpha_{\mathrm{hi}}\}$}

\STATE $\{\tilde{q}^s_{1,\alpha}\}\leftarrow \{\hat{q}_{\alpha}^s+U([-\sigma,\sigma])\}_{i\in \mathcal{I}_1^s} \  \text {and} \  \tilde{q}_{\alpha}(x,s)\leftarrow q_{\alpha}(x,s)+U([-\sigma,\sigma])$

\STATE Compute $\hat{g}_{\alpha}(x, s)\leftarrow\sum_{s^{\prime}=1}^{K} \hat{p}_{s^{\prime}} \hat{F}^{-1}_{2, q_{\alpha}|s^{\prime}} \circ \hat{F}_{1, q_{\alpha}|s} \circ \tilde{q}_{\alpha}({x}, s) $ by Eq. \eqref{eq:sec:cfr:F_hat_2inv} and \eqref{sec:cfr:eq:F_hat_q}
\ENDFOR
\ENDIF
\renewcommand{\algorithmicrequire}{{{\bf Output}:}}
\REQUIRE fair quantile prediction $\hat{g}_{\alpha}$ for calibration set or new test point $(x,s)$.
\end{algorithmic}
\end{algorithm}

\section{Theoretical results}\label{sec:theory}

We provide a statistical analysis of the proposed algorithm with coverage and DP-fairness guarantees in this part.

\begin{theorem}[Prediction coverage guarantee]\label{sec:theory:thm:cfr}
 If $(\tilde{X}_i,Y_i), i = 1, \dots,n+1$ are exchangeable, then the prediction interval $C(\tilde{X}_{n+1})$ constructed by the split CFQP algorithm satisfies $${P}\{Y_{n+1}\in C(\tilde{X}_{n+1})\}\ge 1-\alpha.$$ 
 Moreover, if the conformity scores $E_i$ are almost surely distinct, the prediction interval is nearly exactly calibrated, 
 $${P}\{Y_{n+1}\in C(\tilde{X}_{n+1})\}\le 1-\alpha+1/(|\mathcal{I}_2|+1).$$ 
\end{theorem}

\begin{remark}
Corollary 1 in the supplementary material gives an extension for the conformalization step which allows coverage errors to be spread arbitrarily over the left and right tails. Controlling the left and right tails independently yields a stronger coverage guarantee. 
\end{remark}

\begin{theorem}[Demographic parity guarantee]\label{sec:theory:thm:DP}

For any joint distribution ${P}$ of $({X}, S, Y)$, any $\sigma>0$, as well as the base quantile estimator $\hat{q}_\alpha: \mathbb{R}^{p} \times[K] \rightarrow \mathbb{R}$ constructed on labeled data, the estimator $\hat{g}_\alpha$ defined in Eq. \eqref{sec:cfr:eq:g_hat_test} satisfies
$$
(\hat{g}_\alpha({X}, S) \mid S=s)\overset{d}{=}\left(\hat{g}_\alpha({X}, S) \mid S=s^{\prime}\right) \quad \forall s, s^{\prime} \in[K] .
$$
\end{theorem}

{
Quantile DP guarantees provided by Theorem \ref{sec:theory:thm:DP} are derived directly from distribution-free properties on rank and order statistics in Lemma E.1., Theorem 7.2, of \citet{Chzhen_Schreuder_2022}. Further information can be found in their papers and the references they provide. We extend the estimator $\hat{g}$ for quantile regression based on the estimators developed in their work with solid theoretical foundations. In contrast to the seminal work of \citet{Chzhen_Schreuder_2022}, we regard the densities and quantile functions as functional data, namely, as samples of stochastic processes. Typically, this approach is used by economists when dealing with the densities of income distribution across populations. As a result of introducing the kernel smoothing step, potential performance improvements are demonstrated in the supplemental material.
}

\section{Experiments}\label{sec:experiments}

To evaluate our proposed method \footnote{We utilize the local linear fitting smoothing method in the experiments.}, we report the performance of post-processing fairness adjustment on quantiles through four benchmark datasets: Law School (LAW), Community\&Crime (CRIME), MEPS 2016 (MEPS), Government Salary (GOV). A detailed description of these datasets can be found in the supplementary material. The code for reproducing our results is avaiable at \url{https://github.com/Lei-Ding07/Conformal_Quantile_Fairness}.




\begin{minipage}[!htbp]{\textwidth}
\centering
\scriptsize
 \begin{tabularx}{\textwidth}{XXXXXXXXX}
    \toprule
    \multirow{2}{*}{} &
      \multicolumn{4}{c}{LAW} &
      \multicolumn{4}{c}{CRIME}  \\
      \cmidrule(lr){2-5} \cmidrule(lr){6-9}
      & Coverage & Length & KS(lo) & KS(hi) & Coverage & Length & KS(lo) & KS(hi)\\
    \midrule
  Ln-CQR & 90.16$\pm$0.47 & 0.46$\pm$.004 & 0.39$\pm$0.03 & 0.11$\pm$0.02 & 90.22$\pm$1.88 & 1.30$\pm$0.05 & 0.62$\pm$0.06 & 0.53$\pm$0.06 \\

\textbf{Ln-CFQP}   &90.02$\pm$0.51 & 0.46$\pm$.004  & 0.02$\pm$0.01 & 0.02$\pm$0.01 & 90.44$\pm$1.84 & 1.64$\pm$0.05 & 0.11$\pm$0.03 & 0.12$\pm$0.04 \\
    \midrule
RF-CQR & 90.25$\pm$0.55 & 0.39$\pm$.005 & 0.20$\pm$0.02 & 0.15$\pm$0.02 & 90.27$\pm$1.66 & 1.15$\pm$0.03& 0.64$\pm$0.05 & 0.59$\pm$0.05 \\
\textbf{RF-CFQP}  & 90.11$\pm$0.48 & 0.38$\pm$.004 & 0.02$\pm$.008 & 0.02$\pm$.009 & 90.34$\pm$1.84 & 1.54$\pm$0.04 & 0.12$\pm$0.04 & 0.12$\pm$0.03 \\
\midrule
NN-CQR & 90.00$\pm$0.50 & 0.40$\pm$0.02 & 0.41$\pm$0.07 & 0.18$\pm$0.05 & 90.01$\pm$1.89 & 1.16$\pm$0.05 & 0.70$\pm$0.05 & 0.63$\pm$0.06\\

\textbf{NN-CFQP} & 90.01$\pm$0.51 & 0.39$\pm$0.01 & 0.02$\pm$.009 & 0.03$\pm$.009 & 89.95$\pm$1.62 & 1.54$\pm$0.12 & 0.12$\pm$0.04 & 0.12$\pm$0.03 \\
\midrule

    \multirow{2}{*}{} &
      \multicolumn{4}{c}{MEPS} &
      \multicolumn{4}{c}{GOV} \\
      \cmidrule(lr){2-5} \cmidrule(lr){6-9} 
      & Coverage & Length & KS (lo) & KS(hi)& Coverage & Length & KS (lo) & KS(hi)\\
    \midrule
  Ln-CQR & 89.92$\pm$0.66& 0.66$\pm$0.01 & 0.09$\pm$0.03 & 0.33$\pm$0.05 & 90.00$\pm$0.19 & 0.79$\pm$.002 & 0.26$\pm$.014 & 0.44$\pm$0.02 \\
    \textbf{Ln-CFQP}   & 89.99$\pm$0.69 & 0.66$\pm$0.01  & 0.03$\pm$0.01 & 0.03$\pm$0.01 & 90.02$\pm$0.19 & 0.78$\pm$.002& 0.05$\pm$0.01 & 0.04$\pm$0.01 \\
    \midrule
RF-CQR & 90.07$\pm$0.65 & 0.38$\pm$.009 & 0.19$\pm$0.02 & 0.30$\pm$0.03 & 90.03$\pm$0.17 & 0.61$\pm$.002 & 0.29$\pm$0.01 & 0.28$\pm$0.02 \\
\textbf{RF-CFQP}  & 90.38$\pm$0.60 & 0.39$\pm$0.01 & 0.02$\pm$0.01 & 0.03$\pm$0.01 & 90.03$\pm$0.17 & 0.62$\pm$.002 & 0.05$\pm$0.01 & 0.04$\pm$0.01 \\
\midrule
NN-CQR & 89.95$\pm$0.68 & 0.37$\pm$0.04 & 0.24$\pm$0.09 & 0.37$\pm$ 0.06 & 90.01$\pm$0.19 & 0.58$\pm$0.01 & 0.28$\pm$0.03 & 0.32$\pm$0.04\\

\textbf{NN-CFQP} & 89.97$\pm$0.61 & 0.37$\pm$0.04 & 0.03$\pm$0.01 & 0.04$\pm$0.01 & 90.01$\pm$0.18 & 0.59$\pm$0.01 & 0.05$\pm$0.01 & 0.05$\pm$0.01 \\
\bottomrule
 \end{tabularx}
 \captionsetup{type=table}
\captionof{table}{Results reported on test set of 200 repeated experiments with $\alpha = 0.1$. CQR refers to the conformalized quantile regression in \cite{Romano_Patterson_Candes_2019}. Ln, RF, and NN denote the linear, random forest, and neural network quantile regression models. 
 Our methods are shown in bold.}
\label{sec:exp:tab:result_cqr}
\end{minipage}

\begin{minipage}[!htbp]{\textwidth}
\begin{center}
\captionsetup{type=figure}
\includegraphics[width=\textwidth]{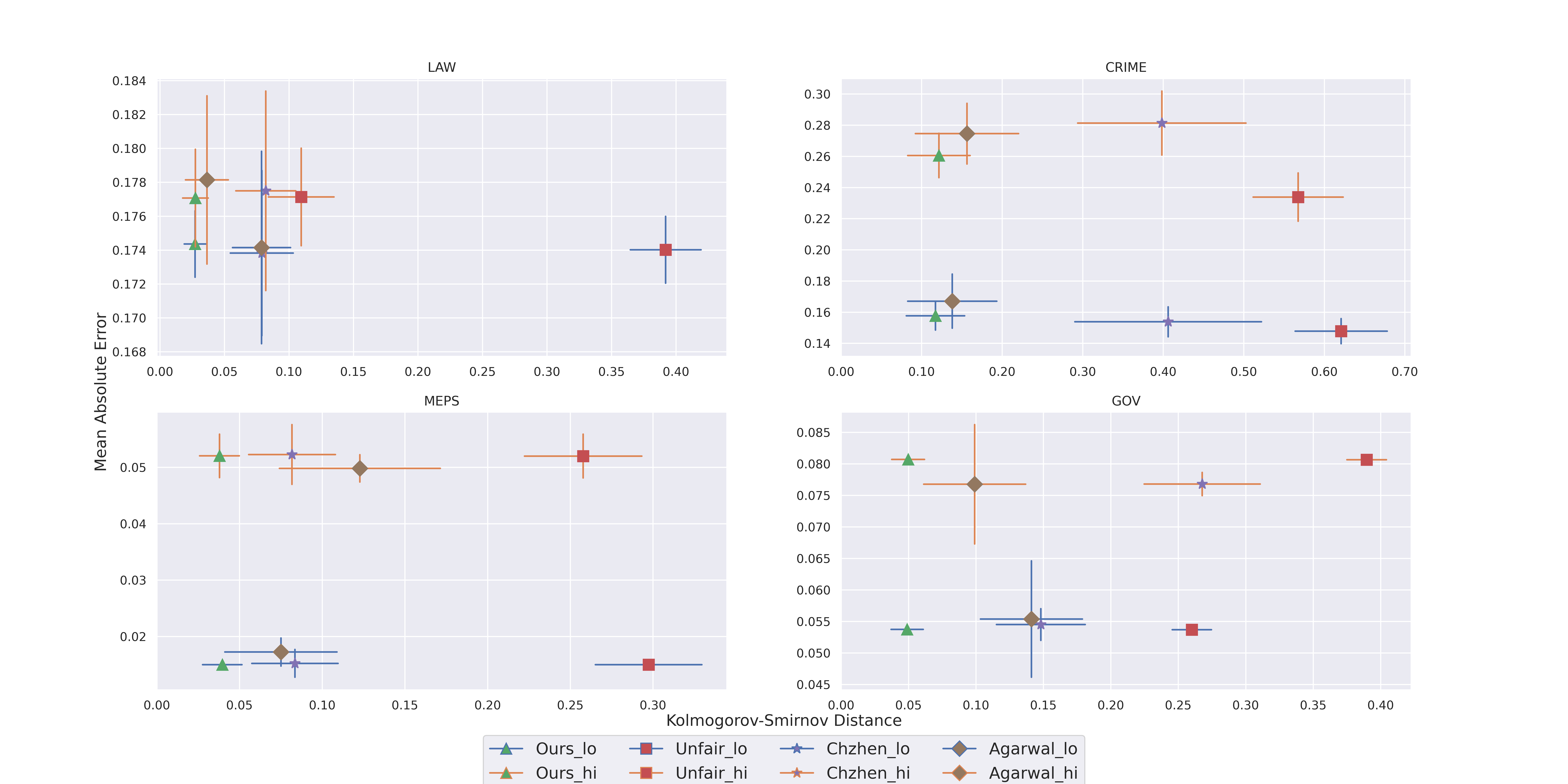}
\captionof{figure}{ Results for estimating the  lower ($\alpha_{\mathrm{lo}}$) and upper ($\alpha_{\mathrm{hi}}$) quantiles using some state-of-the-art DP-fairness requirement methods on all the datasets. `Unfair', `Chzhen', and `Agarwal' stand for the linear quantile model without fairness adjustment, barycenter method \citep{chzhen_denis_hebiri_oneto_pontil_2020} and reduction-based algorithm \citep{Agarwal_Dudik_Wu_2019} respectively. We present the MAE and KS of lower and upper quantile estimations. A Linear quantile model is implemented in this comparison.} 
\label{fig:result:fairness_comparison}
\end{center}
\end{minipage}

We measure the violation of DP-fairness of the quantiles required by Definition \ref{def:Demo_Par}  
through the empirical Kolmogorov-Smirnov (KS) distance. The value represents the disparity between groups $\mathcal{Z}^s= \{(X, S, Y)\in \mathcal{Z}: S=s\}$ for all $s\in [K]$,
\begin{equation*}
     \textstyle{\mathrm{KS}(g_\alpha)=\underset{s,s^{\prime}\in [K]}{\max}\  \underset{t\in\mathbb{R}}{\sup}\left\lvert\frac{1}{|\mathcal{Z}^s|}\underset{(X,S,Y)\in \mathcal{Z}^s}{\sum} \mathbbm{1}{\left\{g_\alpha({X,S})\le t\right\}}-\frac{1}{|\mathcal{Z}^{s^\prime}|}\underset{(X,S,Y)\in \mathcal{Z}^{s^\prime}}{\sum}\mathbbm{1}{\left\{g_\alpha({X,S})\le t\right\}}\right\rvert}.
\end{equation*}

\textbf{Experiment results.}  In Table \ref{sec:exp:tab:result_cqr}, 
we report the average performance of the proposed CFQP over 200 randomly training-test splits as well as the baseline model CQR by the coverage rate, length of prediction interval, and the KS distance of the interval endpoint. We split the training data into proper training and calibration sets of equal sizes. Throughout the experiments, the nominal miscoverage rate is fixed to $\alpha=0.1$. Among pre-existing quantile algorithms, we select three leading variants: \textbf{linear model}\cite{koenker1978regression}, \textbf{random forests} \cite{Meinshausen} and \textbf{neural networks} \cite{taylor2000quantile}. Overall, our CFQP likewise CQR constructs prediction bands attaining desirable coverage around 90\%, as claimed in Theorem \ref{sec:theory:thm:cfr}. Random forest based approaches tend to be slightly more conservative than the other two w.r.t the coverage rate among all four datasets. 

In the KS column concerning the DP fairness of interval endpoints, our CFQP method greatly reduces the discriminatory bias (quantified by KS) by 70\% up to 90\% compared to that of CQR. In addition, the lengths of the prediction intervals mostly remain the same except for the Crime dataset, probably due to  its inherent high discriminatory bias between sensitive groups. In addition, CFQP is more robust according to the standard errors over experiment repetitions.



Figure \ref{fig:result:fairness_comparison} presents the comparison of our post-processing fairness adjustment procedure on quantiles $\hat{g}_{\alpha}$ using the test set $\mathcal{Z}_{test}=\{(X_i,S_i,Y_i)\}_{i=1}^{n_{test}}$ with some state-of-the-art fairness algorithms. Since most of the algorithms are targeted at mean prediction, there is no direct comparison with our quantile fairness method; we accordingly modified the existing methods into quantile versions for comparison. A detailed description can be found in the supplementary material.

The points in Figure \ref{fig:result:fairness_comparison} represents the mean of 200 repeated experiments with $x$-axis as KS distance and $y$-axis as Mean Absolute Error(MAE), where $\operatorname{MAE}(g_\alpha)=\nicefrac{1}{n_{test}}\sum_{\mathcal{Z}_{test}}|Y-g_\alpha(X,S)|$ measures the prediction error of quantiles, the bars are the standard error on both axes. The optimal points should locate at the bottom left corner of the graph, where smaller KS distance and smaller MAE are achieved. In each subplot, our method consistently performs better with the smallest KS distance while keeping the MAE equal or slightly higher than the others or the unfair version.  Note that due to the highly right skewness of real datasets, the MAE of the upper quantile estimation is larger than that of the lower quantile for all quantile approaches.
Additional ablation studies for computational time analysis and testing the kernel smoothing approach are incorporated in the supplementary material.

\section{Conclusion and future work}\label{sec:conclusion}

Conformal fair quantile regression is a novel approach for creating fair prediction intervals that attain valid coverage and reach independence between sensitive attributes while making minimal modifications to the quantile endpoints simultaneously. It becomes superior within heteroskedastic and/or asymmetric datasets and robust to outliers.

Our method is supported by rigorous distribution-free coverage and exact DP-fairness guarantees, as proved in theoretical parts. We conducted several real data examples demonstrating the effectiveness of our method in achieving exact coverage while imposing DP-fairness in practice. The method outperforms several state-of-the-art approaches by comparison.

A limitation in our numerical experiments is that we simply utilize the local linear smoothing method in defining quantile functions of the subgroups; we believe incorporating flexible kernel smoothing approaches \cite{Yang_1985,Zhang_Muller_2011} would improve the experimental performances.

As potential future works, it would be valuable to introduce a DP relaxation framework based on an unfairness measure in a similar manner as \citep{Chzhen_Schreuder_2022,williamson2019fairness}, allowing controlling the level of unfairness in quantile estimates. We also expect to extend the scope to other potential fairness metrics which is dependent on the underlying response like equalizing
quantile loss across groups by incorporating a fairness penalty term in training, or the fairness metric defined for conditional variance-at-risk. In addition, it is worthwhile exploring the quantile fairness in other areas of AI, such as NLP\citep{ding2022word,hu2022balancing}, Computer Vision, Recommendation systems, etc.

\section*{Acknowledgements}
This work was supported by the Economic and Social Research Council (ESRC ES/T012382/1) and the Social Sciences and Humanities Research Council (SSHRC 2003-2019-0003) under the scheme of the Canada-UK Artificial Intelligence Initiative. The project title is BIAS: Responsible AI for Gender and Ethnic Labour Market Equality.
We thank Yang Hu, Xiaojun Du and Matthew Pietrosanu for their helpful discussion and valuable input and all the constructive 
suggestions and comments from the reviewers.


\bibliographystyle{abbrvnat}
\bibliography{bib}

\newpage

\textbf{\large Supplementary Material}
\appendix

\section{Conditions for Proposition 2.}
\emph{Condition (A1).} The quantile density functions $Q^{\prime}_{s^{\prime}}, s^{\prime}=1, \dots, K$, are twice continuously differentiable in $(0,1)$, and satisfy $\inf_{t \in[0,1]} Q^{\prime}_{s^{\prime}}(t) \geq c_{0}>0$, for a constant $c_{0}>0$,  $\forall s^{\prime}$. There is a $\gamma>0$ such that $\sup _{u \in[0,1]} u(1-u)|J_{s^{\prime}} (u)|  \leq \gamma$ for all $s^{\prime}$, where $J_{s^{\prime}}(u):=[\mathrm{d} \log Q^{\prime}_{s^{\prime}}(u) / \mathrm{d} u]$. 

\textit{Condition (A2).} There exists $0<L_{0}<\infty$, such that $\sup _{u \in[0,1]}|\int_{0}^{1} Q^{\prime}_{s^{\prime}}(t) K_h(u-t) \mathrm{d} t| \leq L_{0}, \forall k$.

\textit{Condition (A3).} The kernel functions $K$ is probability density functions which is symmetric around $0$. For any function $f$ that is at least twice continuously differentiable in $(0,1)$,  it holds that for a $\rho>1 / 2$, $\lim \sup _{N \rightarrow \infty} h^{2} N^{\rho}<\infty$, and
$
\sup _{u \in[a, b]\subset(0,1)}\left|f(u)-\int_{0}^{1} f(t) K_h(u-t) \mathrm{d} t\right|=O\left(N^{-\rho}\right),$
where $N$ is defined to ensure that the numbers of measurement asymptotically increases in the same way across the groups, we assume that there exists a sequence $N = N(n)$ with $N\rightarrow\infty,$ as $n\rightarrow\infty$, such that  $\nicefrac{N_{s^{\prime}}}{N}\rightarrow \tau_{s^{\prime}}$ for positive constants, and $0<c_0\le\inf_{1\le s^{\prime}\le K}\tau_{s^{\prime}}\le\sup_{1\le s^{\prime}\le K}\tau_{s^{\prime}}<C_0<\infty, s^{\prime}=1,\dots, K$.

\begin{remark}
Conditions (A1) - (A3) guarantee the existence of a strong approximation of the empirical quantile process by a sequence of weighted Brownian bridges as established in \citep{Miklos_Csorgo1978}. Condition (A3) posited on kernel functions assures that the integral transform $\hat{F}^{-1} \mapsto \hat{Q}$ possesses good approximation properties for smooth functions, and it is shown that (A3) holds for any difference kernel $\mathrm{d}_{t} K_{h}(u, t)=h^{-1} k\left((u-t)/h\right)\mathrm{d}t$ with a vanishing bandwidth $h$. for example, the gaussian density $K(u)=\exp(-{u^{2}}/{2 h^{2}})$ \cite{cheng1997unified,Yang_1985,Zhang_Muller_2011} or the triangular density function $K(u)=(1-|u|/h)I(|u|/h\le1)$ with a vanishing bandwidth $h_n$.
\end{remark}

\section{Proofs related to CFQP}
In this section, we prove the validity of the CFQP prediction intervals described in Section \ref{sec:CFP} of the main paper. First, we recap some results on distribution-free order statistics from \citet{Romano_Barber_Sabatti_Candes_2019}.

The quantile function $Q$ of a random variable $Z$, with cumulative distribution function $F(z):={P}\{Z \leq z\}$, is defined by the equivalence
$$
Q(\alpha) \leq z \quad \text { if and only if } \quad \alpha \leq F(z)
$$
for all $\alpha \in(0,1)$ and $z \in \mathbb{R}$. And less standardly, the right quantile function $R$ of the random variable $Z$ is defined by the equivalence
$F^{-}(z) \leq \alpha \text { if and only if }  z \leq R(\alpha),$
where $F^{-}(z):=F(z-)={P}\{Z<z\}$. The quantile functions have the explicit formulas
$$
Q(\alpha)=\inf \{z \in \mathbb{R}: \alpha \leq F(z)\}, \quad R(\alpha)=\sup \left\{z \in \mathbb{R}: F^{-}(z) \leq \alpha\right\} .
$$
As a special case, the empirical quantile function $\hat{Q}_{n}$ of random variables $Z_{1}, \ldots, Z_{n}$ is the quantile function with respect to the empirical $\operatorname{CDF} \hat{F}_{n}(z):=\frac{1}{n} \sum_{i=1}^{n} \mathbbm{1}\{Z_{i} \leq z\} .$ Likewise, the right empirical quantile function $\hat{R}_{n}$ of $Z_{1}, \ldots, Z_{n}$ is the right quantile function with respect to $\hat{F}_{n}^{-}(z)=\frac{1}{n} \sum_{i=1}^{n} \mathbbm{1}\{Z_{i}<z\}$. They have the explicit formulas
$$
\hat{Q}_{n}(\alpha)=Z_{(\lceil\alpha n\rceil)}, \quad \hat{R}_{n}(\alpha)=Z_{(\lfloor\alpha n\rfloor+1)},
$$

where $Z_{(k)}$ denotes the $k$ th smallest value in $Z_{1}, \ldots, Z_{n}$.
Several variants of the following lemmas are showed in \citet{Vovk_Gammerman_Shafer_2005},\citet{vovk2009line} and \citet{Romano_Patterson_Candes_2019}. We restate them here for clarification.

\begin{lemma}[Quantiles and exchangeability]\label{sec:appendix:lem:1}
Suppose $Z_{1}, \ldots, Z_{n}$ are exchangeable random variables. For any $\alpha \in(0,1)$,
$$
{P}\left\{Z_{n} \leq \hat{Q}_{n}(\alpha)\right\} \geq \alpha .
$$
Moreover, if the random variables $Z_{1}, \ldots, Z_{n}$ are almost surely distinct, then also
$$
{P}\left\{Z_{n} \leq \hat{Q}_{n}(\alpha)\right\} \leq \alpha+\frac{1}{n}.
$$
Here the probabilities are taken over all the variables $Z_{1}, \ldots, Z_{n}$.
\end{lemma}

\begin{lemma}[Inflation of quantiles]\label{sec:appendix:lem:2}
Suppose $Z_{1}, \ldots, Z_{n+1}$ are exchangeable random variables. For any $\alpha \in(0,1)$,
$$
{P}\left\{Z_{n+1} \leq \hat{Q}_{n}\left(\left(1+\frac{1}{n}\right) \alpha\right)\right\} \geq \alpha .
$$
Moreover, if the random variables $Z_{1}, \ldots, Z_{n+1}$ are almost surely distinct, then also
$$
{P}\left\{Z_{n+1} \leq \hat{Q}_{n}\left(\left(1+\frac{1}{n}\right) \alpha\right)\right\} \leq \alpha+\frac{1}{n+1} .
$$
\end{lemma}

Applying the previous auxiliary lemmas, we can prove the validity of the CFQP prediction intervals described in Section \ref{sec:CFP}.

\textbf{Theorem 1.} If $(\tilde{X}_i,Y_i), i = 1, \dots,n+1$ are exchangeable, then the prediction interval $C(\tilde{X}_{n+1})$ constructed by the split CFQP algorithm satisfies $${P}\{Y_{n+1}\in C(\tilde{X}_{n+1})\}\ge 1-\alpha.$$ 
 Moreover, if the conformity scores $E_i$ are almost surely distinct, the prediction interval is nearly exactly calibrated, 
 $${P}\{Y_{n+1}\in C(\tilde{X}_{n+1})\}\le 1-\alpha+1/(|\mathcal{I}_2|+1).$$

\begin{proof}[Proof of Theorem \ref{sec:theory:thm:cfr}.]

Conditionally on the proper training set.
Denote by $E_{n+1}$ the conformity score
$$
E_{i}:=\max \left\{\hat{g}_{\alpha_{\mathrm{lo}}}\left(\tilde{X}_{i}\right)-Y_{i}, Y_{i}-\hat{g}_{\alpha_{\mathrm{hi}}}\left(\tilde{X}_{i}\right)\right\}
$$
at the test point $\left(\tilde{X}_{n+1}, Y_{n+1}\right)$. By the construction of the prediction interval, we have
$$
Y_{n+1} \in C\left(\tilde{X}_{n+1}\right) \quad \text { if and only if } \quad E_{n+1} \leq Q_{1-\alpha}\left(E, \mathcal{I}_{2}\right) \text {, }
$$
and thus,
\begin{equation}\label{sec:appendix:proof:thm1:eq1}
    {P}\left\{Y_{n+1} \in C\left(\tilde{X}_{n+1}\right) \mid\left(\tilde{X}_{i}, Y_{i}\right): i \in \mathcal{I}_{1}\right\}={P}\left\{E_{n+1} \leq Q_{1-\alpha}\left(E, \mathcal{I}_{2}\right) \mid\left(\tilde{X}_{i}, Y_{i}\right): i \in \mathcal{I}_{1}\right\}
\end{equation}

Since the original pairs $\left(\tilde{X}_{i}, Y_{i}\right)$ are exchangeable, so are the calibration variables $E_{i}$ for $i \in \mathcal{I}_{2}$ and $i=n+1$. Thus, by Lemma \ref{sec:appendix:lem:2} on inflated empirical quantiles,
\begin{equation}\label{sec:appendix:proof:thm1:eq2}
    {P}\left\{E_{n+1} \leq Q_{1-\alpha}\left(E, \mathcal{I}_{2}\right) \mid\left(\tilde{X}_{i}, Y_{i}\right): i \in \mathcal{I}_{1}\right\} \geq 1-\alpha,
\end{equation}
and, under the additional assumption that the $E_{i}$ 's are almost surely distinct,
\begin{equation}\label{sec:appendix:proof:thm1:eq3}
    {P}\left\{E_{n+1} \leq Q_{1-\alpha}\left(E, \mathcal{I}_{2}\right) \mid\left(\tilde{X}_{i}, Y_{i}\right): i \in \mathcal{I}_{1}\right\} \leq 1-\alpha+\frac{1}{\left|\mathcal{I}_{2}\right|+1}
\end{equation}

The exact coverage result is derived by taking expectations over the proper training set in \ref{sec:appendix:proof:thm1:eq1}, \ref{sec:appendix:proof:thm1:eq2}, and \ref{sec:appendix:proof:thm1:eq3}. 
\end{proof}

Next, we likewise prove the validity of the extended CFQP prediction intervals that control the left and right tails independently.

\textbf{Corollary 1.} Define the prediction interval $$
C(\tilde{X}_{n+1}):=\left[\hat{g}_{\alpha_{\mathrm{lo}}}\left(\tilde{X}_{n+1}\right)-Q_{1-\alpha_{\mathrm{lo}}}\left(E_{\mathrm{lo}}, \mathcal{I}_{2}\right), \hat{g}_{\alpha_{\mathrm{hi}}}\left(\tilde{X}_{n+1}\right)+Q_{1-\alpha_{\mathrm{hi}}}\left(E_{\mathrm{hi}}, \mathcal{I}_{2}\right)\right]$$
where $Q_{1-\alpha_{\mathrm{lo}}}\left(E_{\mathrm{lo}}, \mathcal{I}_{2}\right)$ is the $\left(1-\alpha_{\mathrm{lo}}\right)$-th empirical quantile of $\left\{\hat{g}_{\alpha_{\mathrm{lo}},i}-Y_{i}: i \in \mathcal{I}_{2}\right\}$ and $Q_{1-\alpha_{\mathrm{hi}}}\left(E_{\mathrm{hi}}, \mathcal{I}_{2}\right)$ is the $\left(1-\alpha_{\mathrm{hi}}\right)$-th empirical quantile of $\left\{Y_{i}-\hat{g}_{\alpha_{\mathrm{hi}},i}: i \in \mathcal{I}_{2}\right\}$. If the samples $\left(\tilde{X}_{n+1}, Y_{i}\right), i=1, \ldots, n+1$ are exchangeable, then
\begin{flalign}
\label{sec:appendix:proof:cor1:eq1}& &{P}\left\{Y_{n+1} \geq \hat{g}_{\alpha_{\mathrm{lo}}}\left(\tilde{X}_{n+1}\right)-Q_{1-\alpha_{\mathrm{lo}}}\left(E_{\mathrm{lo}}, \mathcal{I}_{2}\right)\right\} \geq 1-\alpha_{\mathrm{lo}},&&\\
\label{sec:appendix:proof:cor1:eq2}&\text{and}&{P}\left\{Y_{n+1} \leq \hat{g}_{\alpha_{\mathrm{hi}}}\left(\tilde{X}_{n+1}\right)+Q_{1-\alpha_{\mathrm{hi}}}\left(E_{\mathrm{hi}}, \mathcal{I}_{2}\right)\right\} \geq 1-\alpha_{\mathrm{hi}}.&&
\end{flalign}
Consequently, we also have ${P}\left\{Y_{n+1} \in C\left(\tilde{X}_{n+1}\right)\right\} \geq 1-\alpha$ assuming $\alpha=\alpha_{\mathrm{lo}}+\alpha_{\mathrm{hi}}$.

\begin{proof}
The two events inside the probabilities \eqref{sec:appendix:proof:cor1:eq1} as well as \eqref{sec:appendix:proof:cor1:eq2} are equivalent to $\hat{g}_{\alpha_{\mathrm{lo}}}\left(X_{n+1}\right)-Y_{n+1} \leq$ $Q_{1-\alpha_{\mathrm{lo}}}\left(E_{\mathrm{lo}}, \mathcal{I}_{2}\right)$ and $Y_{n+1}-\hat{g}_{\alpha_{\mathrm{hi}}}\left(X_{n+1}\right) \leq Q_{1-\alpha_{\mathrm{hi}}}\left(E_{\mathrm{hi}}, \mathcal{I}_{2}\right)$, respectively. The results are derives by applying Lemma \ref{sec:appendix:lem:2} twice, in the same manner as in the proof of Theorem \ref{sec:theory:thm:cfr}.
\end{proof}


\section{Proofs related to DP} 

\begin{proof}[Proof of Proposition \ref{sec:cfr:prop:smooth_quantile_est}]

By Theorem 2.1(2) in \citet{cheng1997unified}, we have
\begin{equation}\label{sec:appendix:proof:DP:eq1}
   \sup _{t \in[0,1]}\left|\hat{Q}_{2, q_{\alpha}|s^{\prime}}(t)-{Q}_{q_{\alpha}|s^{\prime}}(t)\right|=O_{p}\left(N^{-1 / 2}+N^{-\rho}\right)=O_{p}\left(N^{-1 / 2}\right), 
\end{equation}
for each $k$, as $\rho>1 / 2$.
according to Assumption (A1) , conditions (Q1)-(Q3) in \citet{cheng1997unified} are satisfied, and since we choose the kernel function $K_h$ with the properties in (A3), conditions (K1)-(K3) in \citet{cheng1997unified} are ensured. Moreover, the extension from $[a, b] \subset(0,1)$ to $[0, 1]$ is made possible by Assumption (A1). At last, assumption (A2) and the fact that the bound in equation (2.7) in \citet{cheng1997unified} is a universal bound which does not depend on $s^{\prime}$ allow the extension from Eq. \ref{sec:appendix:proof:DP:eq1} to $$
\sup _{s^{\prime}} \sup _{t \in[0,1]}\left|\hat{Q}_{2, q_{\alpha}|s^{\prime}}(t)-{Q}_{q_{\alpha}|s^{\prime}}(t)\right|=O_{p}\left(N^{-\nicefrac{1}{2}}\right), \quad s^{\prime}=1, \ldots, K.$$
\end{proof}

{
Before showing the exact DP guarantee, we utilize Lemma E.1. (stated in Lemma 3) from \citet{Chzhen_Schreuder_2022} and references therein, where rigorous proofs are given.
}
\begin{lemma}\label{sec:appendix:lem3}
Let $V_{1}, \ldots, V_{n}, V_{n+1}, n \geq 1$ be exchangeable real-valued random variables and $U$ distributed uniformly on $[0,1]$ be independent from $V_{1}, \ldots, V_{n}, V_{n+1}$, then the constructed location statistic
$$
T\left(V_{1}, \ldots, V_{n}, V_{n+1}, U\right)=\frac{1}{n+1}\left(\sum_{i=1}^{n} \mathbbm{1}\left\{V_{i}<V_{n+1}\right\}+U \cdot\left(1+\sum_{i=1}^{n} \mathbbm{1}\left\{V_{i}=V_{n+1}\right\}\right)\right)
$$
is distributed uniformly on $[0,1]$.

\end{lemma}

{
The proof of Theorem \ref{sec:theory:thm:DP} is a direct adaptation of quantile demographic parity guarantee from \citet{Chzhen_Schreuder_2022} for the mean regression.
}
\begin{proof}[Proof of Theorem \ref{sec:theory:thm:DP}.]
To prove the claimed DP guarantee for fixed quantile level $\alpha\in\{\alpha_\mathrm{lo},\alpha_\mathrm{hi}\}$, we will show that the Kolmogorov-Smirnov distance between $\nu_{\hat{g}_\alpha\mid s}$ and $\nu_{\hat{g}_\alpha\mid s^\prime}$ equals to zero for any $s \neq s^{\prime} \in[K]$.

Note that, according to the formulation of $\hat{g}_\alpha$ in Eq. \eqref{sec:cfr:eq:g_hat_test}, we have for any $(x, s) \in \mathbb{R}^{p} \times[K]$,
$$
 \hat{g}_{\alpha}(x,s)=\sum_{s^{\prime}=1}^{K} \hat{p}_{s^{\prime}} \hat{Q}_{2, q_{\alpha}|s^{\prime}} \circ \hat{F}_{1, q_{\alpha}|s} \circ \tilde{q}_{\alpha}(x,s),\forall \alpha\in\{\alpha_{\mathrm{lo}},\alpha_{\mathrm{hi}}\}.
$$
Denote by $\hat{Q}(t)=\sum_{s^{\prime}=1}^{K} \hat{p}_{s^{\prime}} \hat{Q}_{2, q_{\alpha}|s^{\prime}}$.
 Note that we use the training set to estimate the location statistic for the new test point, $\hat{Q}(t)$ is independent from $\hat{F}_{1, q_{\alpha}|s} \circ \tilde{q}_{\alpha}(x,s)$ for each $s \in[K]$.

Since the test point belongs to group $S=s$ for some fixed $s \in[K]$ and, for all $i=1, \ldots, |\mathcal{I}^s_1|$, set $V_{i}=\tilde{q}_{1,i}^{s}$ with $V_{N_{s}+1} \overset{d}{=}(\tilde{q}_\alpha(x,s))$ independent from $\left(V_{i}\right)_{i=1, \ldots, |\mathcal{I}^s_1|}$. Since the random variables $V_{1}, \ldots, V_{N_{s}}, V_{N_{s}+1}$ are exchangeable (more ideally, independent), Lemma \ref{sec:appendix:lem3} implies that for all $s \in[K]$, the location statistic
$\hat{F}_{1, q_{\alpha}|s} \circ \tilde{q}_{\alpha}(x,s)$ is distributed uniformly on $[0,1]$.
Thus for all $s, s^{\prime} \in[K]$, we have
$$
\begin{aligned}
&\mathrm{KS}\left(\nu_{\hat{g}_\alpha\mid s},\nu_{\hat{g}_\alpha\mid s^\prime}\right) \\
=&\sup _{t \in \mathbb{R}}\left|{P}(\hat{g}_\alpha\leq t \mid S=s)-{P}\left(\hat{g}_\alpha \leq t \mid S=s^{\prime}\right)\right| \\
=&\sup _{t \in \mathbb{R}}\left|{P}\left(\hat{F}_{1, q_{\alpha}|s} \circ \tilde{q}_{\alpha}(x,s) \leq \hat{Q}^{-1}(t) \mid S=s\right)-{P}\left(\hat{F}_{1, q_{\alpha}|s} \circ \tilde{q}_{\alpha}(x,s) \leq \hat{Q}^{-1}(t) \mid S=s^{\prime}\right)\right| \\
=&\sup _{t \in \mathbb{R}}\left|{E}\left[\hat{Q}^{-1}(t) \mid S=s\right]-{E}\left[\hat{Q}^{-1}(t) \mid S=s^{\prime}\right]\right|=0 .
\end{aligned}
$$
The first equality uses the definition of Eq. \eqref{eq:sec:cfr:ghat_est}; the second uses the fact that $\hat{Q}$ is monotone by construction \cite{Chzhen_Schreuder_2022}; finally since the independence of $\hat{Q}$ is independent from $\hat{F}_{1, q_{\alpha}|s} \circ \tilde{q}_{\alpha}(x,s)$ conditionally on $S=s$ for any $s \in[K]$, also $\hat{Q}$ remains independent from $S$. The exact DP is concluded.
\end{proof}

\section{Experiments}

\subsection{Data Description \& Pre-processing}

\begin{enumerate}
\itemsep0em
\item \textit{Law School (LAW)}\cite{wightman1998lsac}: The dataset contains 20,649 examples aiming to predict students' GPA based on their information and capacities, with gender as the sensitive attribute (male vs. female).
 \item \textit{Community\&Crime (CRIME)}\cite{redmond2002data}: This dataset contains socio-economic, law enforcement, and crime data about communities in the US  with 1,994 examples. The task is to predict the number of violent crimes per 100,000 population 
with race as the sensitive attribute. 
\item  \textit{MEPS 2016 (MEPS)}\cite{MEPS_2020,Romano_Patterson_Candes_2019}: The Medical Expenditure Panel Survey 2016 dataset contains information on individuals and their utilization of medical services. The goal is to predict the health care system utilization score of each individual by their features including age, marital status, race, poverty status, health status, health insurance type, and more. 
There are 15,656 examples on $p = 41$ features with race as the sensitive attribute (nonwhite vs. white). 
\item \textit{Government Salary (GOV)}\cite{Plecko_Meinshausen,plevcko2021fairadapt}: The government salary dataset (available in R package "fairadapt") is collected from the 2018 American Community Survey by the US Census Bureau. The yearly salary for over 200,000 examples is
the response variable, and employee race (7 categories) is identified as the sensitive attribute.
\end{enumerate}
In order to smoothly run the regression model, several preprocessing steps are utilized before running the model. For example, in the CRIME dataset, we impute the missing value by mean, and the race feature is created by the maximum value of the four races. The clean data are uploaded to the GitHub repo for future research.

\subsection{Adaptation of other algorithms in \cite{Agarwal_Dudik_Wu_2019,chzhen_denis_hebiri_oneto_pontil_2020} }
 First, our approach is built upon the one proposed by \citet{chzhen_denis_hebiri_oneto_pontil_2020}, we incorporated the kernel smoothing procedure in quantile estimation and applied the local linear smoothing method provided in \emph{NumPy} \cite{harris2020array} which shows its functionality when there are subgroups of small sample sizes, especially for the \textit{CRIME} dataset. It is also feasible to compute Eq. \eqref{eq:sec:cfr:F_hat_2inv} via some global kernels, such as the gaussian density $K(u)=\exp(-{u^{2}}/{2 h^{2}})$ \cite{cheng1997unified,Yang_1985,Zhang_Muller_2011} or the triangular density function $K(u)=(1-|u|/h)I(|u|/h\le1)$ with a vanishing bandwidth $h_n$. For the bandwidth selection, there is a publicly available R package \emph{lokern} with the global bandwidth choice.

Second, we adjusted the reduction-based approach in \citet{Agarwal_Dudik_Wu_2019} in several respects: we rescale and discretize the responses, and modify their algorithm by replacing the loss function $l$ by $\rho_\alpha$ of Eq. \eqref{sec:quant:eq:quant_obj} in our paper. It is worth mentioning that the reduction-based approach is sensitive to the hyperparameters: discretization parameter $N$ and slack $\hat{\varepsilon}_a$ in training. We used logistic regression and SVM classifiers in tuning.

\subsection{Additional experiment results}

We show the additional results using quantile random forest and quantile neural network for the comparison of our post-processing fairness adjustment procedure in figures \ref{fig:result:comparison_rf} and \ref{fig:result:comparison_nn}. The three quantile models present consistent results, our post-processing based upon kernel smoothing outperforms the other two approaches, which is reflected in more KS value reduction. We also conducted a brief ablation study by testing the experimental results of removing either or both of the smoothing strategies in our CFQP approach. Results are presented in table \ref{sec:exp:tab:result_ablation}. We found incorporating the jittering and kernel smoothing methods works better when the subgroups are unbalanced and there exist subgroups of small sample sizes, especially for the \textit{CRIME} dataset.
Finally, the computational time analysis is also incorporated in figure \ref{fig:result:comparison_time}.

\begin{minipage}[!htbp]{\textwidth}
\begin{center}
\captionsetup{type=figure}
\includegraphics[width=\textwidth]{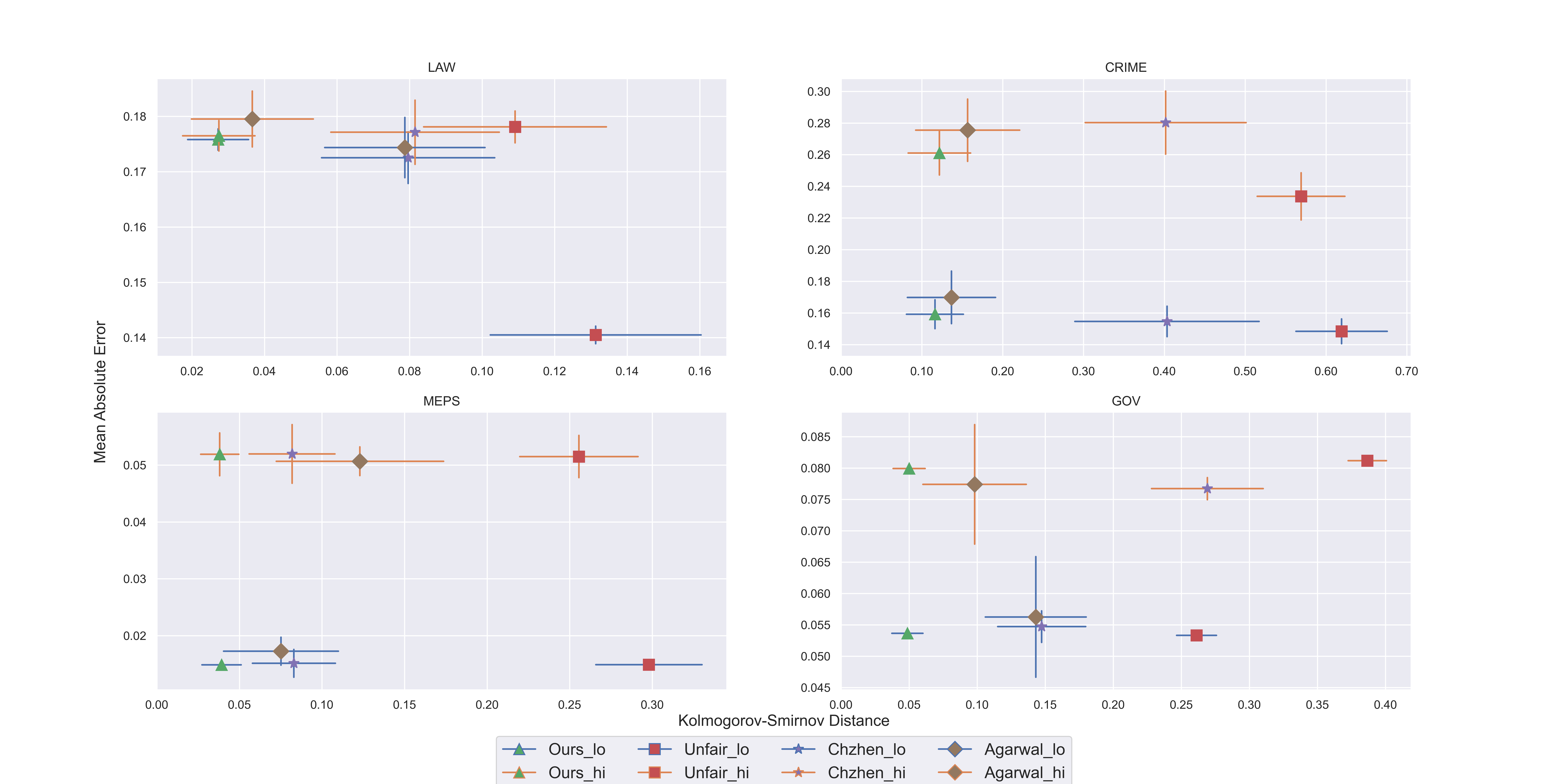}
\captionof{figure}{ Results using Quantile random forest for estimating the lower ($\alpha_{\mathrm{lo}}$) and upper ($\alpha_{\mathrm{hi}}$) quantiles using some state-of-the-art DP-fairness requirement methods on all the datasets. `Unfair', `Chzhen', and `Agarwal' stand for the quantile model without fairness adjustment, barycenter method \citep{chzhen_denis_hebiri_oneto_pontil_2020} and reduction-based algorithm \citep{Agarwal_Dudik_Wu_2019} respectively. We present the MAE and KS of lower quantile estimation, as well as upper quantile estimation. We set `n estimator' to 50.}
\label{fig:result:comparison_rf}
\end{center}
\end{minipage}

\begin{minipage}[!htbp]{\textwidth}
\begin{center}
\captionsetup{type=figure}
\includegraphics[width=\textwidth]{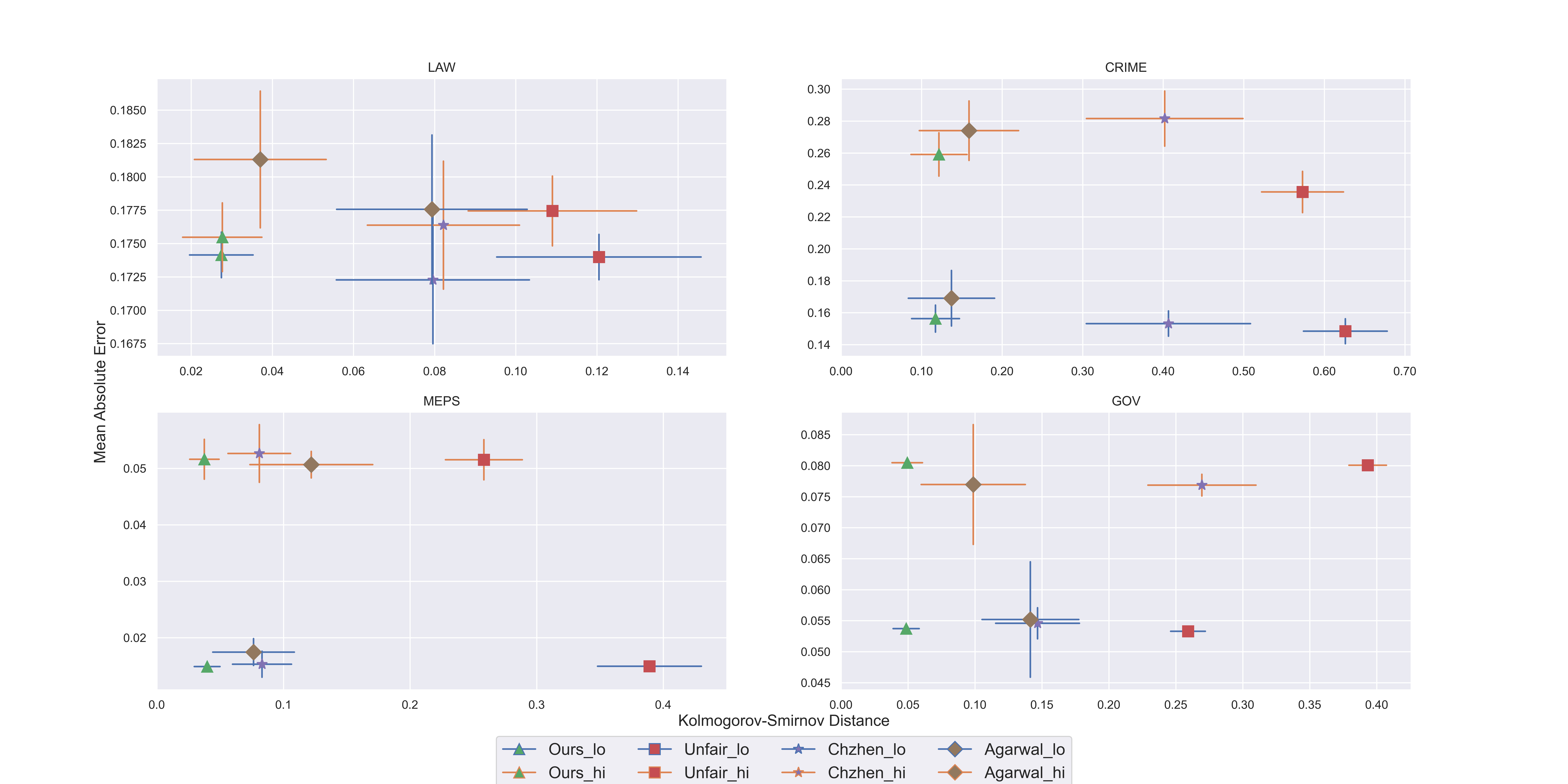}
\captionof{figure}{  Results using Quantile neural network model for estimating the lower ($\alpha_{\mathrm{lo}}$) and upper ($\alpha_{\mathrm{hi}}$) quantiles using some state-of-the-art DP-fairness methods on all the datasets. `Unfair', `Chzhen', and `Agarwal' stand for the  quantile model without fairness adjustment, barycenter method \citep{chzhen_denis_hebiri_oneto_pontil_2020} and reduction-based algorithm \citep{Agarwal_Dudik_Wu_2019} respectively. We present the MAE and KS of lower quantile estimation, as well as upper quantile estimation.}
\label{fig:result:comparison_nn}
\end{center}
\end{minipage}


\begin{minipage}[!htbp]{\textwidth}
\centering
\scriptsize
\scalebox{1.2}{
\begin{tabular}{*{5}{>{\scriptsize}c}}
    \toprule
    \multirow{2}{*}{} &
      \multicolumn{4}{c}{LAW}  \\
      \cmidrule(lr){2-5} 
      & Coverage & Length & KS(lo) & KS(hi) \\
    \midrule

\textbf{Ln-CFQP} (with both smoothing)   &90.02$\pm$0.51 & 0.46$\pm$.004  & 0.02$\pm$0.01 & 0.02$\pm$0.01  \\

  Ln-CFQP (with jittering) & 89.99$\pm$0.49 & 0.39$\pm$.002  & 0.03$\pm$.008 & 0.03$\pm$0.01 \\
  
  Ln-CFQP (with kernel smoothing) & 89.99$\pm$0.50 & 0.38$\pm$.002 & 0.03$\pm$0.009 & 0.03$\pm$0.01 \\

  Ln-CFQP (without smoothing)  & 89.96$\pm$0.49 & 0.40$\pm$.002 & 0.04$\pm$0.01 & 0.03$\pm$0.01  \\
  \midrule
  \multirow{2}{*}{} &
      \multicolumn{4}{c}{CRIME}  \\
      \cmidrule(lr){2-5}
      & Coverage & Length & KS(lo) & KS(hi) \\
    \midrule
\textbf{Ln-CFQP} (with both smoothing)   & 90.44$\pm$1.84 & 1.64$\pm$0.05 & 0.11$\pm$0.03 & 0.12$\pm$0.04 \\

  Ln-CFQP (with jittering)  & 90.55$\pm$1.35 & 1.69$\pm$0.05 & 0.31$\pm$0.11 & 0.33$\pm$0.12 \\
  
  Ln-CFQP (with kernel smoothing)  & 90.58$\pm$1.36 & 1.69$\pm$0.05 & 0.25$\pm$0.13 & 0.29$\pm$0.12 \\

  Ln-CFQP (without smoothing)  & 90.53$\pm$1.27 & 1.69$\pm$0.05 & 0.36$\pm$0.12 & 0.40$\pm$0.12 \\

\bottomrule
\end{tabular}
}
 \captionsetup{type=table}
\captionof{table}{Ablation test results when removing either or both of the smoothing strategies. Our methods are shown in bold.}
\label{sec:exp:tab:result_ablation}
\end{minipage}

In the current experiment, the kernel we used is the local linear one in defining the quantile functions of subgroups instead of the global kernels. When calculating the $\tau$-th quantile $x_\tau$ of $q_\alpha$ using local linear smoothing, we choose a constant distance size $h$ (kernel radius) and compute a weighted average for all data points that are closer to $x_\tau$  (the closer to $x_\tau$ points get higher weights). The time complexity for computing the local kernel smoother (Eq.(8)) is $\mathcal{O}(1)$, while if we applied the global kernels in Eq.(8), $\mathcal{O}(n)$ time would cost.

For the CFQP method, The steps determining the time complexity of Algorithms 1 and 2 reside in the following two parts: 

\begin{enumerate}
    \item The for-loop where we perform a post-processing which takes $\sum_{s^\prime \in [K]} \mathcal{O}\left(N_{s^\prime} \log N_{s^\prime}\right)$ time, as we need to sort the grouped samples; 
    \item The evaluation of $\hat{g}_\alpha$ on a new point $(x, s)$ is performed in $\max _{s^\prime \in [K]} \mathcal{O}(\log N_{s})$ time as it involves locating $\hat{g}_\alpha$ in a sorted array.
\end{enumerate} 

\begin{minipage}[!htbp]{\textwidth}
\begin{center}
\captionsetup{type=figure}
\includegraphics[width=0.7\textwidth]{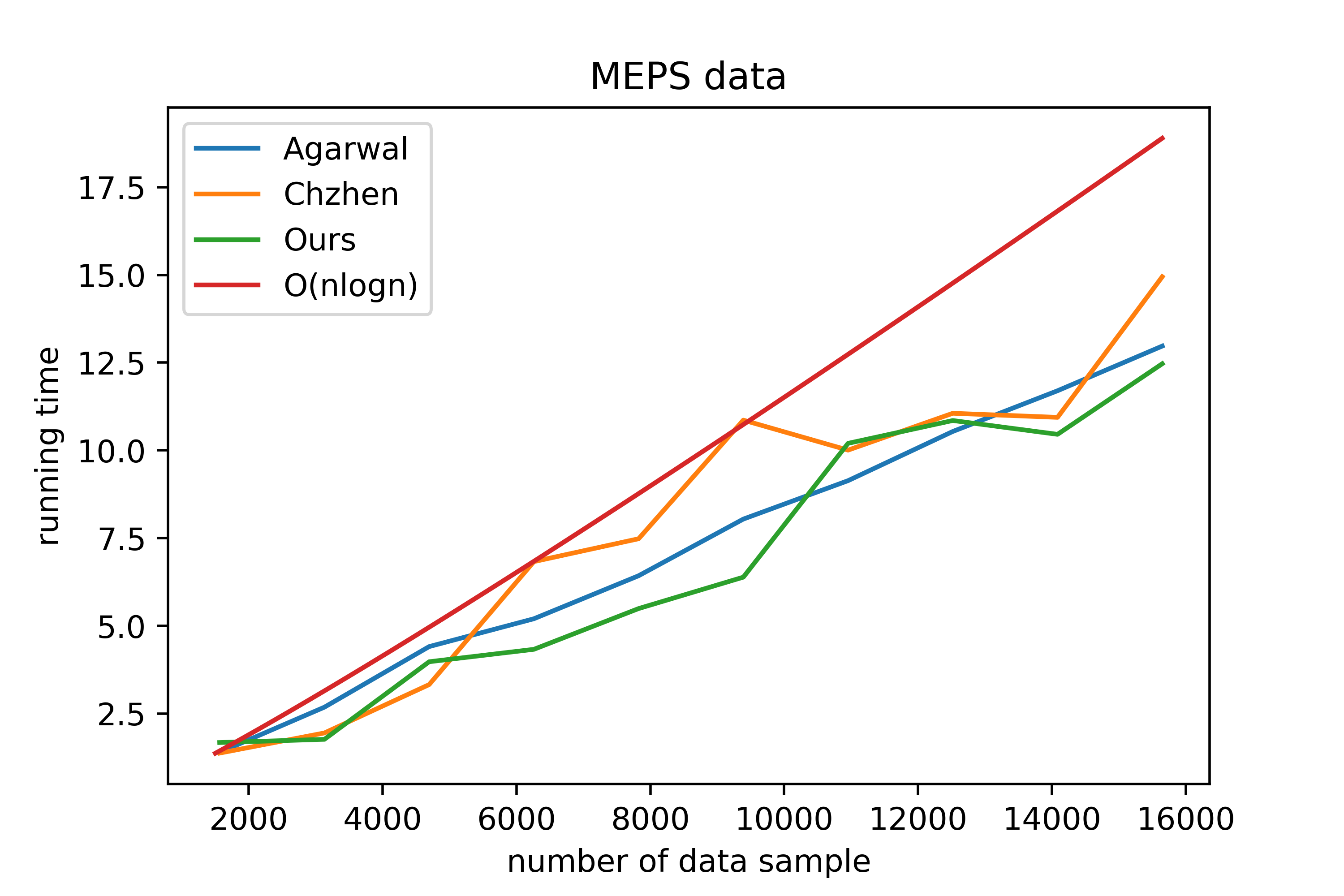}
\captionof{figure}{Empirical running times of methods used for experimental comparisons. We utilized the linear quantile model on the MEPS dataset. For methods of "Chzhen" and "Agarwal", normalizing factors are applied to present the graph.}
\label{fig:result:comparison_time}
\end{center}
\end{minipage}

\end{document}